\documentclass[conference]{IEEEtran}
\usepackage{cite}

\usepackage{hyperref}       
\usepackage{url}            
\usepackage{booktabs}       
\usepackage{amsfonts}       
\usepackage{nicefrac}       
\usepackage{microtype}      
\usepackage{enumerate}
\usepackage{microtype}
\usepackage{graphicx}
\usepackage{subfigure}
\usepackage{booktabs} 
\usepackage{times,epsfig}
\usepackage{epstopdf}
\usepackage{bm}
\usepackage{amssymb,amsmath,amsthm}
\usepackage{threeparttable}
\usepackage{array}
 \usepackage{multirow}
\usepackage{threeparttable}
\newtheorem{definition}{Definition}

\newtheorem{theorem}{Theorem}

\usepackage[table,xcdraw]{xcolor}

\ifCLASSINFOpdf
\else
\fi
\usepackage{algorithmic}
\usepackage{array}
\usepackage{subfigure}

\begin{document}
%
\title{An Accuracy-Lossless Perturbation Method for Defending Privacy Attacks in Federated Learning}


\author{Xue~Yang, Yan~Feng, Weijun~Fang, Jun~Shao,\\
        Xiaohu~Tang,~\IEEEmembership{Member,~IEEE,}
        Shutao~Xia,~\IEEEmembership{Member,~IEEE,}
        Rongxing~Lu,~\IEEEmembership{Fellow,~IEEE}
\thanks{X. Yang, Y. Feng, W. Fang and S. Xia are with Tsinghua Shenzhen International Graduate School, Tsinghua University, and also with the PCL Research Center of Networks and Communications, Peng Cheng Laboratory, Shenzhen, 518055, China.}
\thanks{J. Shao is with the School of Computer and Information Engineering, Zhejiang Gongshang University, Zhejiang, 310018, China.}
\thanks{X. Tang is with the Information Security and National Computing Grid Laboratory, Southwest Jiaotong University, Chengdu, 610031, China.}
\thanks{R. Lu is with the Canadian Institute of Cybersecurity, Faculty of Computer Science, University of New Brunswick, Fredericton, Canada, E3B 5A3.}
\thanks{Corresponding author: W. Fang (e-mail: nankaifwj@163.com)}}


%


\IEEEoverridecommandlockouts
\makeatletter\def\@IEEEpubidpullup{6.5\baselineskip}\makeatother
\IEEEpubid{\parbox{\columnwidth}{
    Network and Distributed Systems Security (NDSS) Symposium 2021\\
    21-24 February 2021\\
    ISBN 1-891562-66-5\\
    https://dx.doi.org/10.14722/ndss.2021.23xxx\\
    www.ndss-symposium.org
}
\hspace{\columnsep}\makebox[\columnwidth]{}}

\maketitle

\begin{abstract}
Although federated learning improves privacy of training data by exchanging local gradients or parameters rather than raw data, the adversary still can leverage local gradients and parameters to obtain local training data by launching reconstruction and membership inference attacks. To defend such privacy attacks, many noises perturbation methods (like differential privacy or CountSketch matrix) have been widely designed. However, the strong defence ability and high learning accuracy of these schemes cannot be ensured at the same time, which will impede the wide application of FL in practice (especially for medical or financial institutions that require both high accuracy and strong privacy guarantee). To overcome this issue, in this paper, we propose  \emph{an efficient model perturbation method for federated learning} to defend reconstruction and membership inference attacks launched by curious clients. On the one hand, similar to the differential privacy, our method also selects random numbers as perturbed noises added to the global model parameters, and thus it is very efficient and easy to be integrated in practice. Meanwhile, the random selected noises are positive real numbers and the corresponding value can be arbitrarily large, and thus the strong defence ability can be ensured. On the other hand, unlike differential privacy or other perturbation methods that cannot eliminate the added noises, our method allows the server to recover the true gradients by eliminating the added noises. Therefore, our method does not hinder learning accuracy at all.  Extensive experiments demonstrate that for both regression and classification tasks, our method achieves the same accuracy as non-private approaches and outperforms the state-of-the-art related schemes. Besides, the defence ability of our method is significantly better than the state-of-the-art related defence schemes. Specifically, for the membership inference attack, our method achieves attack success rate (ASR) of around $50\%$, which is equivalent to blind guessing. However, the ASR of other defence methods is around $60\%$, which  means  that  clients  have  a  certain  advantage  to attack successfully compared with blind guessing. For the reconstruction attack, the ASR of our method is around $10\%$ for 10-classes datasets, which is equivalent to blind guessing (i.e., the optimal defensive ability). However, the ASR of other defensive methods is around $50\%$, which demonstrates a relative poor defensive ability.
\end{abstract}


%

\section{Introduction}
\IEEEPARstart{W}{ith} the continued emergence of privacy breaches and data abuse \cite{Wikipedia18}, data privacy and security issues gradually impede the flourishing development of deep learning \cite{YangLCT19}. In order to mitigate such privacy concerns, \emph{federated learning} (FL) \cite{McMahanMRHA17} has recently been presented as an appealing solution. As illustrated in Fig. \ref{fig:FL}, FL is essentially a distributed learning framework where many clients collaboratively train a shared global model under the orchestration of a central server, while ensuring that each client's raw data is stored locally and not exchanged or transferred. 

\begin{figure}[!t]
\centering
\includegraphics[width=3.5in]{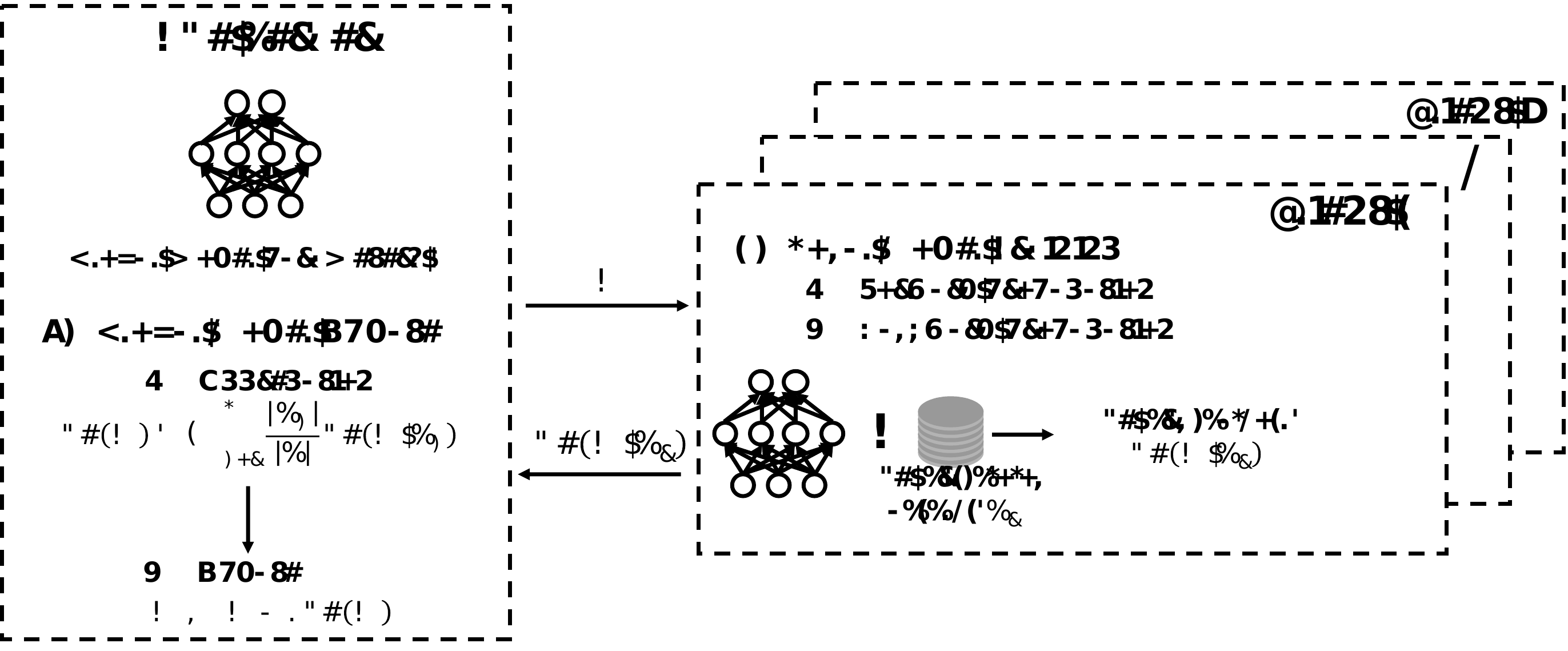}
\caption{The overall framework of federated learning. During the phase of deep model training, clients run stochastic gradient descent (SGD) algorithm on local training data to compute local gradients and upload them. After receiving local gradients from clients, the server aggregates them, updates the current global model parameters, and then distributes the updates global model parameters to clients for the next iteration.}
\label{fig:FL}
\end{figure}

As we all know,  data is the huge digital wealth in recent big data era, and thus all clients (especially for companies) wants to obtain as much training data as possible. And if clients own enough training data, they can train an effective model independently, instead of joining the Fl to share the trained global model. This can totally protect privacy and save huge overheads, especially for communication overhead. FL seemingly can protect the training data of clients by exchanging local gradients and current model parameters rather than raw data, however, sharing local gradients and model parameters is still a well-known privacy risk \cite{ZhuLH19, NasrSH19}. Specifically, inference attacks \cite{NasrSH19} has been widely adopted to reconstruct or infer the property of clients' training data based on the model output. Inference attacks on FL mainly includes two attacks: tracing (a.k.a. membership inference) attacks, and reconstruction attacks. The goal of \emph{reconstruction attacks} is to recover original training data as accurate as possible. The goal of \emph{membership inference attacks} is to infer if a particular individual data record was included in
the training dataset. This is a decisional problem, and its accuracy directly demonstrates the leakage of the model about its training data. Note that we mainly focus on the privacy attacks that will leak the training data of clients, and thus other attacks are beyond the scope of this paper.

To defense inference attacks launched by curious clients, many privacy-preserving federated learning schemes \cite{McMahanRT018, PathakRR10, WeiLDMYFJQP20} widely adopt the \emph{differential privacy} technique \cite{Dwork06} due to its theoretical guarantee of privacy protection and low computational and communication complexity. However, there is a trade-off between privacy guarantee and learning accuracy. If the added noises is strong, the differential privacy works for defence, but the noise inevitably hurts the accuracy and may even stop FL from making progress \cite{HitajAP17}. If the noise is not strong enough, the defence may face failure. In order to overcome such drawback, many related works \cite{FanNJZLCY20, Wangshusen2019, ShokriS15} have been designed. For example, the work in \cite{ShokriS15} has suggested to randomly select and share a small fraction of gradient elements (those with large magnitudes) to reduce privacy loss. Fan et. al. \cite{FanNJZLCY20} have leveraged element-wise adaptive gradients perturbations to defeat reconstruction attacks and maintain high model accuracy. In \cite{Wangshusen2019}, a CountSketch matrix is used to perturb global model parameters for defending reconstruction and membership inference attacks. 

Unfortunately, the defence ability and learning accuracy of these schemes are still unsatisfactory, especially for medical or financial institutions that require both high accuracy and strong privacy guarantee. The main reason is that the perturbed noises added in gradients or model parameters cannot be eliminated, which will inevitably reduce the learning accuracy compared with the training without perturbation. Meanwhile, the defence ability is not strong enough due to the utility consideration in practice (i.e., It should avoid reducing the learning accuracy too much to make the trained model unusable). Therefore, how to design an efficient method in FL that can ensure high defence ability for inference attacks without sacrificing learning accuracy is still a big challenge \cite{abs-1912-04977, LiSTS20}. 

In this paper, we aim to efficiently address such challenge and propose \emph{an efficient model perturbation method for federated learning to defend inference attacks launched by curious clients and ensure high learning accuracy at the same time}. The main novelty and contributions are two-folds:
\begin{itemize}
    \item We propose an efficient model perturbation method to prevent clients obtaining true global model parameters and local gradients, and thus defenses against reconstruction and membership inference attacks \cite{NasrSH19}. On the one hand, similar to the differential privacy, our method also selects random numbers as noises added to the global model parameters, and thus it is very efficient and easy to be integrated. Note that the randomly selected noises are positive real number  and the corresponding value can be arbitrarily large, which shows a strong defensive ability. On the other hand, unlike differential privacy \cite{ShokriS15} or other perturbation methods \cite{FanNJZLCY20, Wangshusen2019} that cannot eliminate the added noises, our method ensures that the server can recover the true gradients by eliminating the added noises. Therefore, our method does not hinder learning accuracy at all.
    \item In order to evaluate our method from \emph{learning accuracy} and \emph{inference attacks defence}, we conduct extensive experiments on several large-sale datasets, and compare our method with several the state-of-the-art related schemes. Empirical results show that our scheme can achieve the same learning accuracy as the scheme that trains model with true parameters (i.e., FedAvg \cite{McMahanMRHA17}), which demonstrates that our method does not harm learning accuracy. Meanwhile, both learning accuracy and defence ability of our method are better than the related defence schemes \cite{ShokriS15, FanNJZLCY20, Wangshusen2019}. More specifically, for the membership inference attack, our method achieves attack success rate of around $50\%$, which is equivalent to blind guessing. However, the attack success rate of other defence methods \cite{ShokriS15, FanNJZLCY20, Wangshusen2019} is around $60\%$, which  means  that  clients  have  a  certain  advantage  to attack successfully compared with blind guessing. For the reconstruction attack, the attack success rate of our method is around $10\%$ for 10-classes datasets, which is equivalent to blind guessing (i.e., the optimal defensive ability). However, the attack success rates of other defensive methods \cite{ShokriS15, FanNJZLCY20, Wangshusen2019} are around $50\%$, which demonstrates a relative poor defensive ability.
\end{itemize}
Note that we admit that our method alone does not fundamentally defend all the attacks launched from curious clients. As claimed in \cite{Wangshusen2019}, there does not exist any defense that is effective for all the attacks that infer privacy. Meanwhile, similar to \cite{ShokriS15, FanNJZLCY20, Wangshusen2019}, we focus on the inference attacks launched by curious clients and do not consider the attacks launched by the server. Actually, we can assume the server is honest and clients are honest-but-curious. To the best of our knowledge, this assumption is widespread in the real word. For example, a insurance headquarter requires all its subsidiaries to join the FL to train a global model for analyzing the insurance behavior of customers. Obviously, the headquarter is trusted by all subsidiaries, but the competing relationships between subsidiaries and they distrust each other, so it should be ensured that none of subsidiaries can obtain the private data of other subsidiaries. If we want to consider both the server and clients are honest-but-curious, our method can be easily incorporated with existing methods such as homomorphic encryption, secret sharing or secure multi-party computing to satisfy the privacy-preserving requirements.

\noindent \textbf{Roadmap} In Section \ref{sec:pre}, we outline preliminaries. In Section \ref{sec:model}, we state the threat model and design goals. We present the proposed scheme in Section \ref{sec:ProSche:MLP}, followed by the theoretical analysis about security in Section \ref{sec:theo}. Experimental results are presented in Section \ref{sec:performance}. Finally we discuss related work in Section \ref{sec:related} and provide concluding remarks in Section \ref{sec:conclusion}.


\section{Preliminaries}\label{sec:pre}
In this section, we outline concepts of the FL and the Hadamard product, which will serve as the basis of our scheme.

\subsection{Federated learning}\label{subsec:pre_FL}
Federated leaning (FL) \cite{abs-1912-04977} is essentially a distributed machine learning framework where a number of clients collaboratively train a shared global model with high accuracy. In this paper, we focus on training Deep Neural Networks (DNNs), which have been widely adopted as the infrastructure of deep learning models to solve many complex tasks \cite{Schmidhuber15}. 

Formally, consider a cross-silo FL with $K$ clients, denoted as $\mathcal{C}=\{\mathcal{C}_{1}, \mathcal{C}_{2}, \ldots, \mathcal{C}_{K}\}$, and each client $\mathcal{C}_{k}$ has the local training dataset $\mathcal{D}_{k}$. The cross-silo FL aims to solve an optimization problem to obtain the optimal global parameter $W$ \cite{McMahanMRHA17, LiHYWZ20}:
\begin{equation}\label{eq:FL}
\min_{W}  F(W) \triangleq \sum_{k=1}^{K}\frac{|\mathcal{D}_{k}|}{|\mathcal{D}|}F(W, \mathcal{D}_{k}),
\end{equation}
where $|\mathcal{D}_{k}|$ is the sample size of the client $\mathcal{C}_{k}$ and $|\mathcal{D}|=\sum^{K}_{k=1}|\mathcal{D}_{k}|$. $F(W, \mathcal{D}_{k})$ is the local object of $\mathcal{C}_{k}$ defined by
\begin{equation}\label{eq:LocalFL}
F(W, \mathcal{D}_{k})\triangleq \frac{1}{|\mathcal{D}_{k}|}\sum_{(\bm{x}_{i}, \bm{\bar{y}}_{i})\in \mathcal{D}_{k}}\mathcal{L}\left(W; (\bm{x}_{i}, \bm{\bar{y}}_{i})\right),
\end{equation}
where $(\bm{x}_{i}, \bm{\bar{y}}_{i})$ is the training sample, $\bm{x}_{i}$ and $\bm{\bar{y}}_{i}$ are the corresponding feature vector and the ground-truth label vector, respectively.
$\mathcal{L}\left(\cdot; \cdot\right)$ is a user-specified loss function, such as the \emph{Mean Squared Error} and \emph{Cross-entropy}. 

In order to find the optimal parameters for the Eq. \eqref{eq:FL}, the server initializes the global model parameter, and then the server and all clients collaboratively perform the DNN training, which mainly includes three phases 1) \emph{global model broadcasting}, 2) \emph{local model training}, and 3) \emph{global model update}. 

\noindent\textbf{1) Global model broadcasting.} The server broadcasts the current global model parameter $W=\{W^{(l)}\}^{L}_{l=1}$ to all clients for local model training.  

\noindent\textbf{2) Local model training.} Once receiving the global model parameter $W=\{W^{(l)}\}^{L}_{l=1}$, each client $\mathcal{C}_{k}$ computes local gradients $\nabla F(W, \mathcal{D}_{k})$ by running the stochastic gradient descent (SGD) algorithm on the local training dataset $\mathcal{D}_{k}$. Generally, for each training sample $(\bm{x}, \bm{\bar{y}}) \in \mathcal{D}_{k}$, the training process includes two steps: \emph{forward propagation} and \emph{backward propagation}. 
\begin{itemize}
\item \textbf{Forward propagation} is the process of obtaining prediction through sequentially calculating under $W$ after the input samples enter the network. Specifically, for the sample feature vector $\bm{x}$, the output vector $\bm y^{(l)}=(y^{(l)}_{1},y^{(l)}_{2},\ldots,y^{(l)}_{n_{l}})\in \mathbb{R}^{n_l}$ of the $l$-th layer is computed as
\begin{equation}\label{eq:forward0}
            \bm y^{(l)}=\left\{
            \begin{aligned}
                & \mathrm{ReLU}\left(W^{(l)}\bm y^{(l-1)} \right), \textnormal{for } 1\leq l\leq L-1, \\
                & W^{(l)}\bm y^{(l-1)} ,~~~~~~~~~~~~ \textnormal{for } l=L, 
            \end{aligned}\right.
        \end{equation}
where $\bm{y}^{(0)}=\bm{x}$ and $\mathrm{ReLU}(\cdot)$ is the widely used activation function satisfying that: for any input $x$,
        \begin{equation*}
            y=\mathrm{ReLU}(x)=\left\{
            \begin{aligned}
                &x, \textnormal{ if } x>0, \\
                &0,\textnormal{ if }x\leq0,
            \end{aligned}\right.
        \end{equation*}
\item \textbf{Backward propagation} starts from the loss value, and updates the parameter values of the network in reverse, so that the loss value of the updated network decreases. Let the gradient of the loss respect to $W^{(l)}$ be $\frac{\partial \mathcal{L}\left(W;(\bm{x}, \bm{\bar{y}})\right)}{\partial W^{(l)}} \in \mathbb{R}^{n_{l}\times n_{l-1}}$. Based on the chain rule, the $(i,j)$-th entry of $\frac{\partial \mathcal{L}\left(W;(\bm{x}, \bm{\bar{y}})\right)}{\partial W^{(l)}}$ is computed as:
      \begin{small}
      \begin{equation*}
         \frac{\partial \mathcal{L}\left(W;(\bm{x}, \bm{\bar{y}})\right)}{\partial w^{(l)}_{ij}} =\frac{\partial \mathcal{L}\left(W;(\bm{x}, \bm{\bar{y}})\right)}{\partial \bm{y}^{(L)}}\frac{\partial \bm{y}^{(L)}}{\partial \bm{y}^{(L-1)}}\cdots \frac{\partial \bm{y}^{(l+1)}}{\partial y^{(l)}_{i}}\frac{\partial y^{(l)}_{i}}{\partial w^{(l)}_{ij}},
     \end{equation*}
     \end{small}
where $i=1,2,\ldots, n_{l}$ and $j=1,2,\ldots, n_{l-1}$. According to Eq. \eqref{eq:forward0}, for $1\leq l\leq L-1$, the corresponding output $\bm y^{(l)} \in \mathbb{R}^{n_{l}}$ is influenced by the non-linear ReLU function, and thus the corresponding gradient $\frac{\partial \mathcal{L}\left(W;(\bm{x}, \bm{\bar{y}})\right)}{\partial W^{(l)}} \in \mathbb{R}^{n_{l}\times n_{l-1}}$ will also be affected by ReLU.
\end{itemize}
Similarly, for every sample ($\bm{x}_{i}, \bm{\bar{y}}_{i})\in \mathcal{D}_{k}$, $\mathcal{C}_{k}$ can obtain the corresponding local gradient $\frac{\partial \mathcal{L}(W; (\bm{x}_{i}, \bm{\bar{y}}_{i}))}{\partial W}$, and then computes the average local gradient as: 
\begin{equation}\label{eq:averlocal}
   \nabla F\left(W, \mathcal{D}_{k}\right)=\frac{1}{|\mathcal{D}_{k}|}\sum_{(\bm{x}_{i}, \bm{\bar{y}}_{i})\in \mathcal{D}_{k}}\frac{\partial \mathcal{L}(W; (\bm{x}_{i}, \bm{\bar{y}}_{i}))}{\partial W}.
\end{equation}
Finally, $\mathcal{C}_{k}$ uploads $\nabla F(W, \mathcal{D}_{k})$ to the server for model update.

\noindent\textbf{3) Global model update.} After receiving $\nabla F(W, \mathcal{D}_{k})$ from all clients, the server aggregates and updates the current model parameter $W$ for the next iteration. Specifically, given the learning rate $\eta$, the updated parameter is computed as
\begin{equation}\label{eq:update}
 W\Leftarrow W-\eta\sum_{k=1}^{K}\frac{|\mathcal{D}_{k}|}{|\mathcal{D}|}\nabla F(W, \mathcal{D}_{k}).
 \end{equation}
After that, the server distributes the updated parameter $W$  to all clients for the next iteration.

\subsection{Hadamard Product}\label{subsec:hadamard}
 The Hadamard product \cite{HJ2012} takes two matrices of the same dimensions and produces another matrix of the same dimension as the operands.
 \begin{definition}\label{hadamard}
 For two matrices $A$ and $B$ of the same dimension $m\times n$, the Hadamard product $A\circ B$ (or $A \odot B$)  is a matrix of the same dimension as the operands, with elements given by
 \[(A \circ B)_{ij}=(A \odot B)_{ij}=A_{ij}B_{ij}\]
 \end{definition}
Two properties of Hadamard product are given as follows:
\begin{itemize}
  \item For any two matrices $A$ and $B$, and a diagonal matrix $D$, we have
  \begin{equation}\label{eq:hadamard0}
  D(A \circ B)=(DA)\circ B \textnormal{ and } (A \circ B)D=(AD)\circ B.
\end{equation}
  \item For any two column vectors $\bm a$ and $\bm b$, the Hadamard product is $\bm{a} \circ \bm{ b}=D_{\bm a} \bm b$,
  where $D_{\bm a}$ is the corresponding diagonal matrix with the vector $\bm a$ as its main diagonal.
\end{itemize}

\section{Threat model and design goals}\label{sec:model}
In this paper, we mainly focus on the privacy attacks launched by clients. Similar to \cite{FanNJZLCY20, Wangshusen2019}, we assume clients are honest-but-curious, which means that all clients honestly follow the underlying scheme and does not submit any malformed messages, but attempt to infer other clients' data privacy by launching privacy attacks. As introduced in Section \ref{subsec:pre_FL}, during each iteration, each client $\mathcal{C}_{k}$ can obtain the updated model parameter $W$ and compute the sum of other clients' gradients denoted as $\sum_{i\neq k}\nabla F(W, \mathcal{D}_{k})$. Knowing the model parameters, gradients or both, curious  clients may try to infer data privacy of other clients. Similar to \cite{FanNJZLCY20, Wangshusen2019}, we mainly focus on two state-of-the-art inference attacks in federated learning called \textbf{reconstruction attack} and \textbf{membership inference attack}. In the \textbf{reconstruction attack}, the curious client tries to reconstruct sensitive features of the records  in  the training set of other clients. In the \textbf{membership inference attack}, the curious client attempts to infer if a certain data record is included in the training dataset with the model parameters. This is a decisional problem, and its accuracy directly demonstrates the leakage of the model about the corresponding training data. Note that, since the privacy preservation is our focus, other active attacks (like poisoning attacks or backdoor attacks) that will destroy the integrity or availability are beyond the scope of this work and will be discussed in the future.

\noindent \textbf{Design goals.} In order to defense privacy attacks launched by honest-but-curious clients and ensure the utility of training model, the design goals of our scheme mainly include the following two aspects:
\begin{itemize}
  \item \textbf{Confidentiality}: The proposed scheme should ensure the curious clients cannot infer private training data of other clients by reconstruct and membership inference attacks with exchanged global model parameters and local gradients. In particular, similar to \cite{Wangshusen2019}, the goal of our scheme is to ensure that clients can only know the random perturbed model parameters rather than true model parameters. 
 
  \item \textbf{High accuracy}: As demonstrated in \cite{abs-1912-04977, FanNJZLCY20}, many state-of-the-art differential privacy-based approaches strengthen the privacy preservation capability in FL at the expense of accuracy, which will hinder the application of FL in practice, especially for medical or financial institutions that require high model accuracy. Therefore, in addition to privacy-preservation, ensuring high model accuracy is also our design goal. More specifically, it is better for our scheme to achieve the same model accuracy as the plain training model, e.g., FedAvg \cite{McMahanMRHA17}.
\end{itemize}

\section{Our proposed Method}\label{sec:ProSche:MLP}
In this section, we describe our proposed privacy-preserving deep model training with ReLU non-linear activation in details, which can be easily applied on the state-of-the-art models such as Convolutional Neural Networks (CNN) \cite{LinRM18} as well as ResNet \cite{HeZRS16} and DenseNet \cite{HuangLMW17}. According to Section \ref{subsec:pre_FL}, the overall framework of our proposed method, as shown in Fig. \ref{fig_framework}, mainly includes the three steps: global model perturbation, local model training and global model update.
\begin{figure}[!t]
\centering
\includegraphics[width=3.5in]{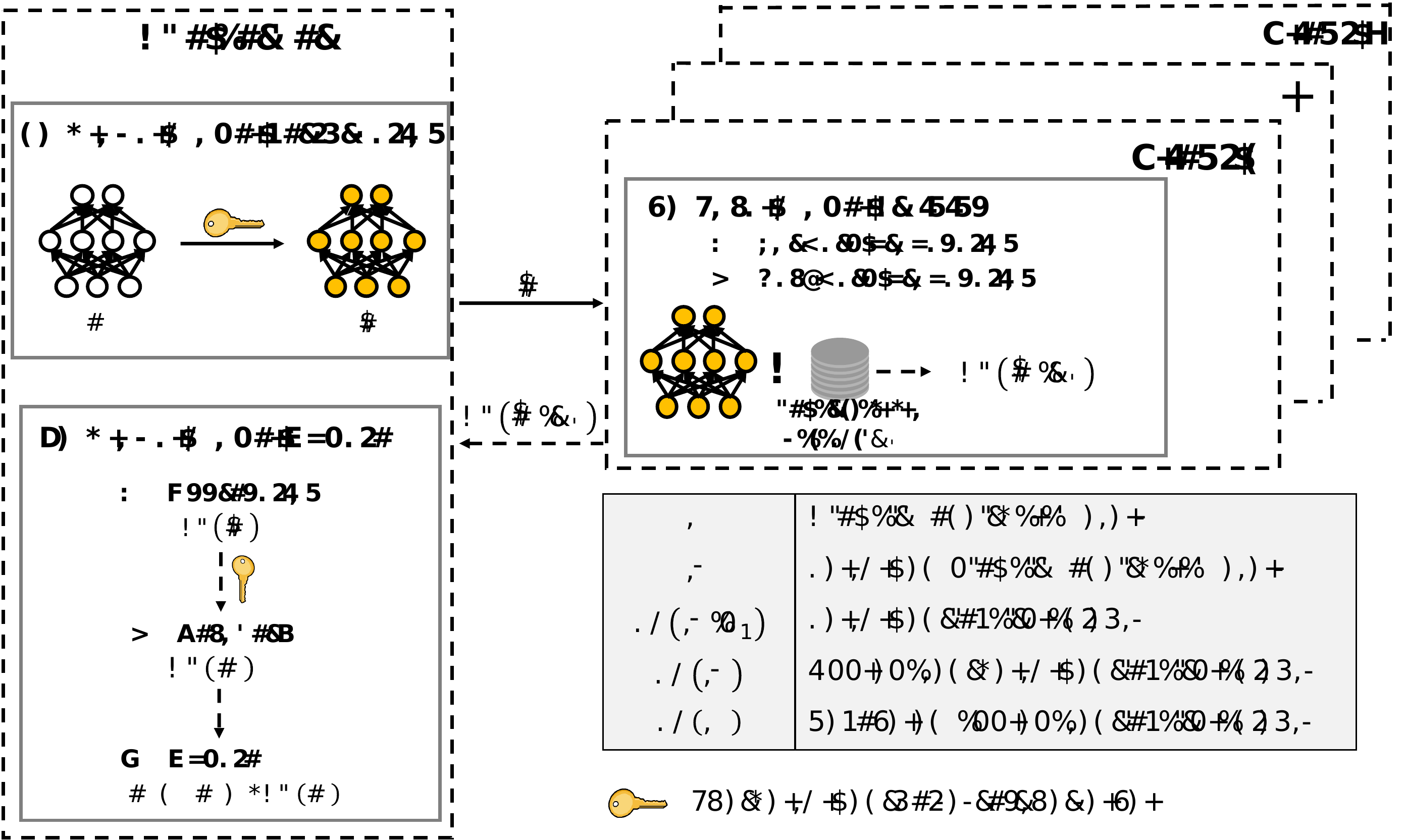}
\caption{The overall framework of our proposed method.}
\label{fig_framework}
\end{figure}

\subsection{Global Model Perturbation}\label{MLP:subsec:ParaPerturb}
The goal of \emph{global model perturbation} is to prevent curious clients from obtaining the true model parameters $W=\{W^{(l)}\}^{L}_{l=1}$ and gradient $\nabla F(W, \mathcal{D}_{k})$. Specifically, the server first needs to select random one-time-used noises for different iterations, and then perturbs global parameters for privacy preservation.
\begin{itemize}
  \item \textbf{Random noises selection:} The server selects random one-time-used noises  for different iterations as follows.
  \begin{itemize}
      \item Randomly select multiplicative noisy vectors $\bm r^{(l)}=(r^{(l)}_{1},r^{(l)}_{2},\ldots, r^{(l)}_{n_{l}}) \in \mathbb{R}^{n_l}_{> 0 }$ for $1\leq l\leq L-1$ and an additive noisy vector $\bm r^{(a)}=(r^{(a)}_{1},r^{(a)}_{2},\ldots,r^{(a)}_{n_{L}}) \in \mathbb{R}^{n_L}$ with pairwise different components.
      \item Define a disjoint partition $\sqcup_{s=1}^{m}\{I_s\}$ of $\{1,2,\dots, n_L\}$ \footnote{The value of $m$ ($1\leq m\leq n_{L}$) determines the defense of predictions, which will be analyzed in Theorem \ref{trade-off}.} such that $\cup_{s=1}^{m}I_s=\{1,2,\dots,n_L\}$ and $I_i \cap I_j=\emptyset$ for any $i \neq j$. Randomly select noisy numbers  $\gamma_{I_1}, \gamma_{I_2}, \dots, \gamma_{I_m} \in \mathbb{R}$. Then let $\bm \gamma=(\gamma_1,\gamma_2,\ldots,\gamma_{n_{L}})$, whose coordinates are given as $\gamma_i=\gamma_{I_s}$ for $i \in I_s$ and $s=1,2, \dots, m$.
  \end{itemize}
Note that, the server will keep $\left(\{\bm r^{(l)}\}_{l=1}^{L-1}, \bm{\gamma}\right)$ secret for preventing curious clients from obtaining true model parameters.
    \item \textbf{Parameter perturbation:} With the random selected noises, the global parameters $W=\{W^{(l)}\}^{L}_{l=1}$ are perturbed as: for the $l$-th layer's parameter matrix $W^{(l)}\in \mathbb{R}^{n_l\times n_{l-1}}$, the perturbed value $\widehat{W}^{(l)}$ is 
   \begin{equation}\label{eq:para_cons}
\widehat{W}^{(l)}=\left\{
\begin{aligned}
&  R^{(l)} \circ W^{(l)} ,  \textnormal{ for } 1\leq l\leq L-1,\\
& R^{(l)}  \circ W^{(l)} + R^{(a)},  \textnormal{ for } l =L,
\end{aligned}\right.
\end{equation}
where $R^{(l)}\in \mathbb{R}^{n_l\times n_{l-1}}$ and $R^{(a)}\in \mathbb{R}^{n_{L}\times n_{L-1}}$ satisfy
 \begin{eqnarray}
R^{(l)}_{ij}&=&\left\{
\begin{aligned}
  & r^{(1)}_i ,  ~~~~~~~~\textnormal{ when } l=1 \\
&  r^{(l)}_i /  r^{(l - 1)}_j, \textnormal{ when } 2\leq l \leq L-1 \\
 & 1 /  r^{(L - 1)}_j,  ~~\textnormal{ when } l=L
\end{aligned}\right.  \label{eq:para_cons1} \\
R^{(a)}_{ij}& =& \gamma_i \cdot r^{(a)}_{i}, \label{eq:para_cons2}
\end{eqnarray}
where  $i\in [1, n_{l}]$ and $j\in [1,n_{l-1}]$ in Eq. \eqref{eq:para_cons1}, and $i\in [1, n_{L}]$ and $j\in [1,n_{L-1}]$ in Eq. \eqref{eq:para_cons2}.
\end{itemize}

Finally, the server broadcasts the current perturbed parameters $\widehat{W}=\{\widehat{W}^{(l)}\}_{l=1}^{L}$ and the additive noisy vector $\bm r^{(a)}$ to all clients for local model training.

\subsection{Local Model Training}\label{MLP:subsec:LocalTraining}
After receiving $\widehat{W}=\{\widehat{W}^{(l)}\}_{l=1}^{L}$, each client $\mathcal{C}_{k} \in \mathcal{C}$ conducts local model training with local training data $\mathcal{D}_{k}$. As introduced in Section \ref{subsec:pre_FL}, for each sample $(\bm{x}_{i},\bm{\bar{y}}_{i}) \in \mathcal{D}_{k}$, $\mathcal{C}_{k}$ executes \emph{forward propagation} and \emph{backward propagation} to compute gradients. Note that we ignore the subscript $i$ of $(\bm{x}_{i},\bm{\bar{y}}_{i})$ for simplicity unless other specification.

\textbf{(1) Forward Propagation.} 
Obviously, with the perturbed parameters $\widehat{W}$, the output of each layer is also perturbed, denoted as $\bm{\hat{y}}^{(l)}$ for $1\leq l\leq L$. Based on Eq. \eqref{eq:forward0}, the perturbed output $\bm{\hat{y}}^{(l)} \in \mathbb{R}^{n_{l}}$ is computed as 
\begin{equation}\label{eq:noisyforward}
\bm{\hat{y}}^{(l)}=\left\{
\begin{aligned}
 &  \mathrm{ReLU}\left(\widehat{W}^{(l)}\bm{\hat{y}}^{(l-1)} \right), \textnormal{ for } 1\leq l\leq L-1. \\
 & \widehat{W}^{(l)}\bm{\hat{ y}}^{(l-1)} , ~~~~~~~~~~~~\textnormal{ for } l=L.
\end{aligned}\right.
\end{equation}
where $\bm{\hat{y}}^{(0)}=\bm{x}$ is the input feature of the sample $(\bm{x},\bm{\bar{y}})$. Theorem \ref{theorem:output} shows the important relations between the perturbed output and the true output given in Eq. \eqref{eq:forward0}.
\begin{theorem}\label{theorem:output}
For $1 \leq l \leq L$, the perturbed output vector $\bm{\hat{y}}^{(l)}$ computed in Eq. \eqref{eq:noisyforward} and the true output vector $\bm{y}^{(l)}$ computed in Eq. \eqref{eq:forward0}
have the following relationship:
\begin{align}
\bm{\hat{y}}^{(l)}= & \bm r^{(l)} \circ \bm y^{(l)}, \textnormal{ when } 1\leq l\leq L-1. \label{eq:l_out} \\
\bm{\hat{y}}^{(L)} = & \bm y^{(L)} + \alpha \bm \gamma \circ \bm r^{(a)}=\bm y^{(L)} +  \alpha \bm{r} . \label{eq:final_out}
\end{align}
where $\alpha=\sum_{i=1}^{n_{L-1}} \hat {y}^{(L - 1)}_i$ and $\bm{r}=\bm \gamma \circ \bm{r}^{(a)}$.
\end{theorem}
\begin{proof}
See Appendix \ref{proof_theory0}.
\end{proof}

\textbf{(2) Backward Propagation.}
After obtaining the perturbed outputs $\{\bm{\hat{y}}^{(l)}\}^{L}_{l=1}$, the client calculates corresponding gradients based on the specific loss function. Next, we take the mean squared error loss function as an example to show the relationship of the perturbed gradients and true gradients, which will help understand why the server can exactly recover the true global model in Section \ref{subsec:update}. Specifically, since the client can only obtain the perturbed prediction $\hat{\bm y}^{(L)}=\widehat{W}^{(L)}\bm{\hat{ y}}^{(L-1)}$, he/she can only use $\hat{\bm y}^{(L)}$ to compute the loss value, which is defined as: 
\begin{equation}\label{noisy loss function}
  \mathcal{\widehat{L}}\left(\widehat{W}; (\bm{x},\bm{\bar{y}})\right)=\frac{1}{2}\parallel\hat{\bm y}^{(L)}- \bar{\bm y}\parallel^{2}_{2}.
\end{equation}
With the above perturbed loss function, each client $\mathcal{C}_{k}$ can compute the perturbed gradient of the $l$-th layer, denoted as $\frac{\partial \mathcal{\widehat  L}(\widehat{W}; (\bm{x},\bm{\bar{y}}))}{\partial \widehat {W}^{(l)}} \in \mathbb{R}^{n_{l}\times n_{l-1}}$, based on the calculations of backward propagation in Section \ref{subsec:pre_FL}. What's more, the following Theorem \ref{gradient} shows the important relationship between the perturbed gradients and the true gradients.

 \begin{theorem}\label{gradient}
 For any $1 \leq l \leq L$, the perturbed gradients $\frac{\partial \mathcal{\widehat L}\left(\widehat{W}; (\bm{x},\bm{\bar{y}})\right)}{\partial \widehat{W}^{(l)}} \in \mathbb{R}^{n_{l}\times n_{l-1}}$ and the true gradients  $\frac{\partial \mathcal{L}\left(W; (\bm{x},\bm{\bar{y}})\right)}{\partial W^{(l)}} \in \mathbb{R}^{n_{l}\times n_{l-1}}$ satisfy
 \begin{equation*}
\frac{\partial \mathcal{\widehat  L}(\widehat{W}; (\bm{x},\bm{\bar{y}}))}{\partial \widehat {W}^{(l)}}=\frac{1}{R^{(l)}} \circ \frac{\partial \mathcal{L}\left(W; (\bm{x},\bm{\bar{y}})\right)}{\partial {W}^{(l)}} +
  \bm{r}^{T} \bm \sigma^{(l)} - \upsilon \bm \beta^{(l)},
\end{equation*}
 where $\frac{1}{R^{(l)}}$, $\bm{r}$, $\bm \sigma^{(l)}$, $\upsilon$ and $\bm \beta^{(l)}$ are denoted as 
 \begin{itemize}
     \item $\frac{1}{R^{(l)}}$ is the $n_l \times n_{l-1}$ matrix whose $(i,j)$-th entry is $\frac{1}{R^{(l)}_{ij}}$;
     \item $\bm{r}=\bm \gamma \circ \bm{r}^{(a)}$ and $\upsilon=\bm{r}^{T}\bm{r}$;
     \item   $\bm \sigma^{(l)} = \alpha\frac{\partial \hat {\bm y}^{(L)}}{\partial \widehat {W}^{(l)}}+ \Big(\frac{\partial \mathcal{\hat L}(\widehat{W} (\bm{x},\bm{\bar{y}}))}{\partial \hat{\bm y}^{(L)}}\Big)^{T}\frac{\partial  \alpha}{\partial \widehat {W}^{(l)}}$;
     \item  $\bm \beta^{(l)} = \alpha \frac{\partial \alpha}{\partial \widehat {W}^{(l)}}$.
 \end{itemize}
 Note that, $\alpha=\sum_{i=1}^{n_{L-1}} \hat {y}^{(L - 1)}_i$, and when $l=L$, $\alpha$ is not a function of $\widehat {W}^{(L)}$, and thus $\frac{\partial  \alpha}{\partial \widehat {W}^{(L)}}=\bm 0_{n_{L}\times n_{L-1}}$, which implies that $\bm \sigma^{(L)} = \alpha\frac{\partial \hat {\bm y}^{(L)}}{\partial \widehat {W}^{(l)}}$ and $\bm \beta^{(L)} = \bm{0}_{n_{L}\times n_{L-1}}$.
 \end{theorem}
\begin{proof}
See Appendix \ref{proof_Theory1}.
\end{proof}

From Theorem \ref{gradient}, we can observe that $\bm \sigma^{(l)}$ and $\bm \beta^{(l)}$ can be computed directly by clients and both two values decide whether the server can recover the true aggregated model parameters (see Section \ref{subsec:update} for more details). Hence, in addition to the perturbed gradients, the client also needs to compute two noisy items $\bm \sigma^{(l)}$ and $\bm \beta^{(l)}$.

Hence, for any sample ($\bm{x}_{i}, \bm{\bar{y}}_{i})\in \mathcal{D}_{k}$, each client $\mathcal{C}_{k}$ computes the perturbed local gradients $\frac{\partial \mathcal{\widehat{L}}(\widehat{W}; (\bm{x}_{i}, \bm{\bar{y}}_{i}))}{\partial \widehat{W}^{(l)}}$ and the corresponding noisy items, represented as 
$\bm{\sigma}^{(l)}_{(\bm{x}_{i}, \bm{\bar{y}}_{i})}$ and $\bm{\beta}^{(l)}_{(\bm{x}_{i}, \bm{\bar{y}}_{i})}$. Then, $\mathcal{C}_{k}$ computes the average local gradients $\nabla F(\widehat{W}^{(l)}, \mathcal{D}_{k})$ and $(\bm{\sigma}^{(l)}_{k}, \bm{\beta}^{(l)}_{k})$ as: for $l=1,2,\ldots, L$,
\begin{equation*}
\left\{
\begin{aligned}
\nabla F(\widehat{W}^{(l)}, \mathcal{D}_{k}) & :=\frac{1}{|\mathcal{D}_{k}|}\sum_{(\bm{x}_{i}, \bm{\bar{y}}_{i})\in \mathcal{D}_{k}}\frac{\partial \mathcal{\widehat{L}}\left(\widehat{W}; (\bm{x}_{i}, \bm{\bar{y}}_{i})\right)}{\partial \widehat{W}^{(l)}}; \\
\bm{\sigma}^{(l)}_{k}&:=\frac{1}{|\mathcal{D}_{k}|}\sum_{(\bm{x}_{i}, \bm{\bar{y}}_{i})\in \mathcal{D}_{k}}\bm{\sigma}^{(l)}_{(\bm{x}_{i}, \bm{\bar{y}}_{i})};\\
\bm{\beta}^{(l)}_{k}&:=\frac{1}{|\mathcal{D}_{k}|}\sum_{(\bm{x}_{i}, \bm{\bar{y}}_{i})\in \mathcal{D}_{k}}\bm{\beta}^{(l)}_{(\bm{x}_{i}, \bm{\bar{y}}_{i})}.
\end{aligned}\right.
\end{equation*}
Note that $\bm \gamma=(\gamma_1,\gamma_2,\ldots,\gamma_{n_{L}})$ is known for the server and $\gamma_i=\gamma_{I_s}$ for $i \in I_s$ and $s=1,2, \dots, m.$ (See ``Random noises selection"). To recover the term $\bm{r}^{T} \bm \sigma^{(l)}$ of the equation in Theorem \ref{gradient}, the client $\mathcal{C}_{k}$ should compute the following items:  
\[\widetilde{\bm \sigma}^{(l)}_{k,s}:=  (\bm r^{(a)}_{|I_s})^{T}\bm{\sigma}^{(l)}_{{k}_{| I_s}}, \textnormal{ for } s=1,2, \dots, m,\]
where $\bm r^{(a)}_{|I_s}$ is the restriction of the vector $\bm r^{(a)}$ on $I_s$ and $\bm{\sigma}^{(l)}_{{k}_{| I_s}}$ is the sub-matrix of $\bm{\sigma}^{(l)}_{k}$ consists of the rows indexed by $I_s$.

Finally, $\mathcal{C}_{k}$ uploads $\{\nabla F(\widehat{W}^{(l)}, \mathcal{D}_{k})\}^{L}_{l=1}$ and $\{\widetilde{\bm \sigma}^{(l)}_{k,1}, \widetilde{\bm \sigma}^{(l)}_{k,2}, \cdots, \widetilde{\bm \sigma}^{(l)}_{k,m},\bm{\widehat{ \beta}}^{(l)}_{k} \}^{L}_{l=1}$ to the server for model update.

\subsection{Global Model Update}\label{subsec:update}
After receiving $\{\nabla F(\widehat{W}^{(l)}, \mathcal{D}_{k})\}^{L}_{l=1}$ and $(m+1)$ noisy terms $\{\widetilde{\bm \sigma}^{(l)}_{k,1}, \cdots, \widetilde{\bm \sigma}^{(l)}_{k,m},\bm{ \beta}^{(l)}_{k} \}^{L}_{l=1}$ from all clients, the server needs to update the current global model. Specifically, in addition to aggregate all received local gradients, the server also needs to recover exact aggregated results to ensure training accuracy. The details are given as follows.
\begin{itemize}
  \item For $l=1,2,\ldots, L$, perform the aggregation operation as:
  \begin{equation}\label{eq:aggre00}
\left\{
\begin{aligned}
\nabla F(\widehat{W}^{(l)})&=\sum^{K}_{k=1}\frac{|\mathcal{D}_{k}|}{|\mathcal{D}|}\nabla F (\widehat{W}^{(l)}, \mathcal{D}_{k}),\\
\bm \sigma^{(l)}_{s}&=\sum^{K}_{k=1}\frac{|\mathcal{D}_{k}|}{|\mathcal{D}|}\widetilde{\bm \sigma}^{(l)}_{k,s}, s=1,2,\dots,m,\\
\bm{ \beta}^{(l)}&=\sum^{K}_{k=1}\frac{|\mathcal{D}_{k}|}{|\mathcal{D}|}\bm{\beta}^{(l)}_{k}.
\end{aligned}\right.
\end{equation}

  \item According to Theorem \ref{gradient}, recover the true aggregated gradients with the secret noises $R^{(l)}$, $\gamma_{I_s}$ and $\upsilon$:
  \begin{small}
  \begin{equation*}
      \nabla F(W^{(l)})=R^{(l)}\circ\left(\nabla F(\widehat{W}^{(l)})-\Big(\sum_{s=1}^{m}\gamma_{I_s}\bm \sigma^{(l)}_{s}\Big)+\upsilon \bm{\beta}^{(l)}\right),
  \end{equation*}
  \end{small}
where $l=1,2,\ldots, L$. 
\item Update the current global model for the next iteration as: for $l=1,2,\ldots,L$,
\begin{equation}\label{eq:aggre}
    W^{(l)}\Leftarrow W^{(l)}-\eta \nabla F(W^{(l)}).
\end{equation}
\end{itemize}
Next, we theoretically demonstrate that the server can obtain the true global model from the perturbed gradients uploaded by clients. In other words, our proposed scheme can achieve the same model accuracy as the true training introduced in Section \ref{subsec:pre_FL}. 
\begin{theorem}\label{modelupdate}
$\nabla F(W^{(l)})$ is the true aggregated gradients satisfying 
\begin{equation}\label{eq:aggregradient}
    \nabla F(W^{(l)})=\sum^{K}_{k=1}\frac{|\mathcal{D}_{k}|}{|\mathcal{D}|}\nabla F\left(W^{(l)}, \mathcal{D}_{k}\right) 
\end{equation}
where $1\leq l \leq L$ and $\nabla F\left(W^{(l)}, \mathcal{D}_{k}\right)$ is the true local gradients denoted in Eq. \eqref{eq:averlocal}.
\end{theorem}
\begin{proof}
See Appendix \ref{proof_Theory2}. 
\end{proof}

According to Theorem \ref{modelupdate}, Eq. \eqref{eq:aggre} is derived as 
\[
W^{(l)}\Leftarrow W^{(l)}-\eta \sum^{K}_{k=1}\frac{|\mathcal{D}_{k}|}{|\mathcal{D}|}\nabla F\left(W^{(l)}, \mathcal{D}_{k}\right),
\]
which is equal to Eq. \eqref{eq:update}, i.e., the update without noises perturbation.

\section{Theoretical analysis}\label{sec:theo}
Based on design goals, we theoretically analyze the inference attacks defence of our method. 
Essentially, we need to ensure that clients cannot obtain the true information received from the server or obtained from local training. As introduced in Section 
\ref{sec:ProSche:MLP}, clients first receive perturbed global model parameters $\widehat{W}$ from the server, and then perform forward propagation and backward propagation to obtain the output $\hat{\bm{y}}^{(l)}$ and the gradient $\frac{\partial \mathcal{\widehat  L}(\widehat{W}; (\bm{x},\bm{\bar{y}}))}{\partial \widehat {W}^{(l)}}$ of each layer, respectively. Therefore, our method needs to meet the following three privacy requirements: (1) the privacy-preservation of global model parameters, (2) the privacy-preservation of outputs in forward propagation and (3) the privacy-preservation of local gradients in backward propagation. 

\subsection{Privacy-preservation of global model parameters}\label{subsec:security:parameter}
As introduced in Section \ref{MLP:subsec:ParaPerturb}, the server perturbs current global model parameters by randomly selected positive real number before distributing. Thus, we prove that curious clients cannot obtain true global model parameters $W$ from the received perturbed parameters $\widehat{W}$. 

\begin{theorem}\label{model parameters}
For any given perturbed global model parameters $\widehat {W}=\{\widehat W^{(l)}\}^{L}_{l=1}$, there always exist infinitely many different model parameters $W=\{ W^{(l)}\}^{L}_{l=1}$ and noises $(\{ R^{(l)}\}_{l=1}^{L},  R^{(a)})$ satisfying Eq \eqref{eq:para_cons}, i.e.,
\begin{equation*}
\widehat{W}^{(l)}=\left\{
\begin{aligned}
&  R^{(l)} \circ W^{(l)} ,  \textnormal{ for } 1\leq l\leq L-1,\\
& R^{(l)}  \circ W^{(l)} + R^{(a)},  \textnormal{ for } l =L.
\end{aligned}\right.
\end{equation*}
\end{theorem}

\begin{proof}
 As shown in Section \ref{MLP:subsec:ParaPerturb}, the  noises $R^{(l)}\in \mathbb{R}^{n_l\times n_{l-1}}$ and $R^{(a)}\in \mathbb{R}^{n_{L}\times n_{L-1}}$ used for perturbing are given as:
  \begin{eqnarray*}
R^{(l)}_{ij}&=&\left\{
\begin{aligned}
  & r^{(1)}_i ,  ~~~~~~~~\textnormal{ when } l=1 \\
&  r^{(l)}_i /  r^{(l - 1)}_j, \textnormal{ when } 2\leq l \leq L-1 \\
 & 1 /  r^{(L - 1)}_j,  ~~\textnormal{ when } l=L
\end{aligned}\right.   \\
R^{(a)}_{ij}& =& \gamma_i \cdot r^{(a)}_{i}. 
\end{eqnarray*}
where the vectors $\bm r^{(l)}=(r^{(l)}_{1},r^{(l)}_{2},\ldots, r^{(l)}_{n_{l}}) \in \mathbb{R}^{n_l}_{> 0 }$ for $l=1, 2, \dots, L-1$, and $\bm r^{(a)}=(r^{(a)}_{1},r^{(a)}_{2},\ldots,r^{(a)}_{n_{L}}) \in \mathbb{R}^{n_L}$ with pairwise different components are randomly selected and kept secret by the server.

Next, we show that we can find infinite pairs $W=\{W^{(l)}\}^{L}_{l=1}$ and $(R=\{R^{(l)}\}^{L}_{l=1}, R^{(a)})$ to construct a given $\widehat {W}=\{\widehat W^{(l)}\}^{L}_{l=1}$. Specifically, for any given perturbed global model parameters $\widehat {W}=\{\widehat W^{(l)}\}^{L}_{l=1}$, based on the above equations, we can construct the model parameters $W=\{ W^{(l)}\}^{L}_{l=1}$ as follows:  
\begin{equation}\label{W}
\left\{
\begin{aligned}
 &  W^{(l)}=\frac{1}{R^{(l)}}\circ\widehat W^{(l)}, \textnormal{ for } l=1,2,\dots, L-1,\\
 & W^{(L)}=\frac{1}{R^{(L)}}\circ(\widehat W^{(L)}-R^{(a)}),
\end{aligned}\right.
\end{equation}
where for $l=1,2, \dots, L$, $\frac{1}{R^{(l)}}\in \mathbb{R}^{n_l\times n_{l-1}}$ is defined as: the $(i,j)$-th entry of $\frac{1}{R^{(l)}}$ is 
$\frac{1}{R^{(l)}}_{ij}=(R^{(l)}_{ij})^{-1}$.

Then, we can show that the noises $(\{R^{(l)}\}_{l=1}^{L}, R^{(a)})$ and the model parameters $W=\{ W^{(l)}\}^{L}_{l=1}$ constructed by Eq. \eqref{W} satisfy Eq. \eqref{eq:para_cons} as follows:
\begin{equation*}
   R^{(l)} \circ W^{(l)} = R^{(l)} \circ\frac{1}{R^{(l)}}\circ\widehat W^{(l)}=\widehat W^{(l)}, ~\textnormal{ for }1\leq l\leq L-1,
\end{equation*}
and 
\begin{eqnarray*}
R^{(L)}  \circ W^{(L)} + R^{(a)} &=& R^{(L)}\circ\frac{1}{R^{(L)}}\circ(\widehat W^{(L)}-R^{(a)})+ R^{(a)} \\
&=& \widehat W^{(L)}-R^{(a)}+ R^{(a)}=\widehat W^{(L)}.
\end{eqnarray*}
 Due to the arbitrariness of the vectors $\{\bm r^{(l)}\}_{l=1}^{L-1}$ and $\bm r^{(a)}$, there are infinitely many $\{ R^{(l)}\}_{l=1}^{L},  R^{(a)}$ since they are determined by  the vectors $\{\bm r^{(l)}\}_{l=1}^{L-1}$ and $\bm r^{(a)}$. Additionally, from Eq. \eqref{W}, we obtain that there exist infinitely many different model parameters $W=\{ W^{(l)}\}^{L}_{l=1}$ and noises $(\{R^{(l)}\}_{l=1}^{L}, R^{(a)})$ satisfying Eq. \eqref{eq:para_cons}.
\end{proof}
According to Theorem \ref{model parameters}, given the perturbed global model parameters $\widehat {W}=\{\widehat W^{(l)}\}^{L}_{l=1}$, there exist infinite pairs $W=\{ W^{(l)}\}^{L}_{l=1}$ and the corresponding noises $(R=\{R^{(l)}\}^{L}_{l=1}, R^{a})$ satisfying Eq. \eqref{eq:para_cons}. Thus, clients cannot determine the true model parameters $W=\{ W^{(l)}\}^{L}_{l=1}$ used for generating $\widehat {W}$.

\subsection{Privacy-preservation of outputs in forward propagation}\label{subsec:security:output}
As shown in \cite{NasrSH19}, curious clients can take the final prediction or intermediate output of models as input to launch membership inference attacks. Thus, we prove that curious clients cannot obtain the true output of each layer in forward propagation.

As shown in Theorem \ref{theorem:output}, the client can only obtain the perturbed outputs $\bm \hat{\bm y}=\{\bm \hat{\bm y}^{(l)}\}^{L}_{l=1}$, which satisfy that 
\begin{equation*}
\left\{
\begin{aligned}
 \bm{\hat{y}}^{(l)}= & \bm r^{(l)} \circ \bm y^{(l)}, \textnormal{ when } 1\leq l\leq L-1, \\
 \bm{\hat{y}}^{(L)} = & \bm y^{(L)} + \alpha \bm \gamma \circ \bm r^{(a)}=\bm y^{(L)} +  \alpha \bm{r},
\end{aligned}\right.
\end{equation*}
where $\alpha=\sum_{i=1}^{n_{L-1}} \hat {y}^{(L - 1)}_i$ and $\bm{r}=\bm \gamma \circ \bm{r}^{(a)}$. Next, we show that clients cannot determine the true output $\bm y^{(l)}$ from the perturbed output $\bm \hat{y}^{(l)}$.

Specifically, for $1\leq l \leq L-1$, since the noise vector $\bm r^{(l)}=(r^{(l)}_{1},r^{(l)}_{2},\ldots, r^{(l)}_{n_{l}}) \in \mathbb{R}^{n_l}_{> 0 }$ is randomly chosen by the server, similar to Theorem \ref{model parameters}, the client obviously cannot determine the true output $\bm y^{(l)}$ from the perturbed output $\bm{\hat{y}}^{(l)}=  \bm r^{(l)} \circ \bm y^{(l)}$.

For the prediction vector $\bm{\hat y}^{(L)}=\bm{y}^{(L)}+\alpha \bm \gamma \circ \bm{r}^{(a)}$, the parameter $\alpha$ and $\bm{r}^{(a)}=(r^{(a)}_{1}, r^{(a)}_{2}, \ldots, r^{(a)}_{n_{L}})$ are known to  clients, but $\bm \gamma=(\gamma_{1}, \gamma_{2}, \ldots, \gamma_{n_{L}})$ is chosen by the server randomly and is unknown to clients. Thus, we show that clients have no advantage to guess the true prediction from $\bm{\hat y}^{(L)}$ without knowing $\bm \gamma$. Recall that there exists a partition $\sqcup_{s=1}^{m}\{I_s\}$ of $\{1,2,\dots, n_L\}$, such that for any $i,j$ in the same $I_s$
satisfy that $\gamma_i=\gamma_j$. Thus, the privacy of predictions is mainly influenced by the parameter $m$. 
Specifically, we give the privacy of predictions in terms of $m$ under our proposed method in the following Theorem.
\begin{theorem}\label{trade-off}
When the number of classes $n_{L}=1$, clients cannot obtain any information about the true prediction $y^{(L)}$; When $n_{L} \geq 2$ and $m=1$,  the probability that clients obtain the true predictions $\bm{y}^{(L)}$ from the perturbed predictions $\bm{\hat y}^{(L)}$ is less than 1; When $n_{L} \geq 2$ and $2 \leq m \leq n_L$, the corresponding probability is less than or equal to $1/m$.
\end{theorem}
\begin{proof}
When $n_{L}=1$, then the prediction is one dimensional, denoted as $\hat y^{(L)}=y^{(L)}+\alpha \gamma  r^{(a)}$, which usually represents regression tasks. Since $\gamma$ is chosen randomly by the server, clients cannot know the true prediction $y^{(L)}$.

 When $n_{L} \geq 2$, then the prediction is multi-dimensional, denoted as $\bm{\hat y}^{(L)}=\bm{y}^{(L)}+\alpha \bm \gamma \circ \bm{r}^{(a)}$, which usually represents classification tasks. Under this case, if $m=1$, then $\bm \gamma$ satisfies $\gamma_1=\gamma_2=\cdots=\gamma_{n_L}=\gamma$. Thus, $\hat{y}^{(L)}_{i}$ is denoted as $\hat y^{(L)}_i=y^{(L)}_i+\alpha \gamma  r^{(a)}_{i}$ for $i=1,2, \dots, n_L$. Note that $r^{(a)}_{1},r^{(a)}_{2}, \dots, r^{(a)}_{n_L}$ are pairwise distinct and $\gamma$ is randomly selected by the server, thus it can not determine the largest one among $y^{(L)}_{1}, y^{(L)}_{2}, \dots, y^{(L)}_{n_L}$. Hence the probability that clients obtain the true prediction is obviously less than $1$.

 If $2\leq m\leq n_{L}$,  then $\bm \gamma$ satisfies $\gamma_i=\gamma_{I_s}$ for $1 \leq s \leq m$ and $i \in I_s$. Let $s' \in I_s$ such that $y_{s'}^{(L)}=\max\limits_{i \in I_s}\{y^{(L)}_i\}$, then $\max\limits_{1 \leq i \leq n_{L}}\{ y^{(L)}_i\}=\max\{y_{1'}^{(L)}, y_{2'}^{(L)}, \dots, y_{m'}^{(L)}\}$, i.e., the maximal one among $ y_{1}^{(L)},  y_{2}^{(L)}, \dots, y_{n_L}^{(L)}$ is equal to the maximal one among $ y_{1'}^{(L)},  y_{2'}^{(L)}, \dots, y_{m'}^{(L)}$.
Hence the probability that clients obtain the true prediction is less than or equal to the probability that clients obtain the maximal one among $ y_{1'}^{(L)},  y_{2'}^{(L)}, \dots, y_{m'}^{(L)}$. Concretely, the noisy prediction vector $(\hat y_{1'}^{(L)}, \hat y_{2'}^{(L)}, \dots, \hat y_{m'}^{(L)})$ and the true prediction vector $ (y_{1'}^{(L)},  y_{2'}^{(L)}, \dots, y_{m'}^{(L)})$ satisfy the following $m$ equations:
\begin{equation}\label{31}
\left\{
\begin{aligned}
\hat y_{1'}^{(L)}&=y_{1'}^{(L)}+\alpha\gamma_{I_1}r^{(a)}_{1'},\\
\hat y_{2'}^{(L)}&=y_{2'}^{(L)}+\alpha\gamma_{I_2}r^{(a)}_{2'},\\
& \vdots \\
\hat y_{m'}^{(L)}&= y_{m'}^{(L)}+\alpha\gamma_{I_m}r^{(a)}_{m'}.
\end{aligned}\right.
\end{equation}
Note that $(\hat y_{1'}^{(L)}, \hat y_{2'}^{(L)}, \dots, \hat y_{m'}^{(L)})$ and $\alpha$ are known to the clients, and $\gamma_{I_1},  \gamma_{I_2}, \dots,  \gamma_{I_m}$ are independent randomly chosen by the server. Thus for any $m$-tuple $( y_{1'}^{(L)}, y_{2'}^{(L)}, \dots,  y_{m'}^{(L)}) \in \mathbb{R}^{m}$, there always exists an $m$-tuple $( \gamma_{I_1}, \gamma_{I_2}, \dots, \gamma_{I_m})$ satisfying Eq. \eqref{31}. Hence the client can not obtain any information about the true prediction vector $ (y_{1'}^{(L)},  y_{2'}^{(L)}, \dots, y_{m'}^{(L)})$. So the probability that the clients obtain the maximal one among $ y_{1'}^{(L)},  y_{2'}^{(L)}, \dots,  y_{m'}^{(L)}$ is equal to $1/m$. Therefore  the probability that clients obtain the true predictions $\bm{y}^{(L)}$ from the perturbed predictions $\bm{\hat y}^{(L)}$ is less than or equal to $1/m$.
\end{proof}
According to the above proofs, clients cannot obtain the correct prediction $\bm y^{(L)}$ of a given sample feature $\bm x$. Hence, they cannot efficiently launch the membership inference attack to infer if a certain data records was part of training dataset.

\subsection{Privacy-preservation of gradients in backward propagation}
As demonstrated in \cite{ZhuLH19}, local gradients can be used to reconstruct training data. Thus, we need to show that our method can efficiently prevent curious clients from obtaining the true local gradients. 

Based on Theorem \ref{gradient}, the perturbed gradients $\frac{\partial \mathcal{\widehat L}\left(\widehat{W}; (\bm{x},\bm{\bar{y}})\right)}{\partial \widehat{W}^{(l)}}$ and the true gradients  $\frac{\partial \mathcal{L}\left(W; (\bm{x},\bm{\bar{y}})\right)}{\partial W^{(l)}}$ satisfy the following condition: 
 \begin{equation*}
\frac{\partial \mathcal{\widehat  L}(\widehat{W}; (\bm{x},\bm{\bar{y}}))}{\partial \widehat {W}^{(l)}}=\frac{1}{R^{(l)}} \circ \frac{\partial \mathcal{L}\left(W; (\bm{x},\bm{\bar{y}})\right)}{\partial {W}^{(l)}} +
  \bm{r}^{T} \bm \sigma^{(l)} - \upsilon \bm \beta^{(l)},
\end{equation*}
where $\frac{\partial \mathcal{\widehat  L}(\widehat{W}; (\bm{x},\bm{\bar{y}}))}{\partial \widehat {W}^{(l)}}$, $\bm \sigma^{(l)}$ and $\bm \beta^{(l)}$ are computed by clients. However, $R^{(l)}$, $\bm{r}^{T}$ and $\upsilon$ are chosen by the server, which are unknown to the client. Similar to the proofs in Sections \ref{subsec:security:parameter} and \ref{subsec:security:output}, it is obviously that clients cannot obtain the true gradient from perturbed local gradients.

\section{Performance Evaluation}\label{sec:performance}

In this section, we empirically evaluate our method on real-world datasets in terms of \textbf{learning accuracy} and \textbf{inference attacks defence} (i.e., reconstruction and membership inference attacks). Meanwhile, we give a comparison with the state-of-the-art schemes to show the advantages of our method.

\subsection{Experimental Setup} 
Our method is implemented based on the native network layers of PyTorch. The experiments are conducted on single Tesla M40 GPU. In order to show that our method would not decrease the learning accuracy and can efficiently defense inference attacks, we adopt the state-of-the-art schemes, i.e., the FedAvg \cite{McMahanMRHA17}, PPDL\cite{ShokriS15}, SPN \cite{FanNJZLCY20} and DBCL \cite{Wangshusen2019}, as the baseline, and implement them based on their official released code. In all experiments, the training epochs and the batch-size of each client are set to be $200$ and $32$, respectively. 


\noindent\textbf{Datasets and Metrics}. We evaluate our method on two privacy-sensitive datasets covering both the bank and medical scenarios and one image dataset. 
\begin{itemize}
    \item \textbf{UCI Bank Marketing Dataset (UBMD)} \cite{MoroCR14} is related to direct marketing campaigns of a Portuguese banking institution and aims to predict the possibility of clients for subscribing deposits. It contains $41188$ instances of $17$ dimensional bank data. Following conventional practise, we split the dataset into training/validation/test sets by 8:1:1. We adopt MSE as the evaluation metric.
    \item \textbf{Lesion Disease Classification (LDC)} \cite{abs-1803-10417} \cite{abs-1902-03368}  provides $8$k training and $2$k test skin images for the classification of lesion disease. We down-sample the images into $64\times 64$ and adopt classification accuracy as the evaluation metric.
    \item \textbf{CIFAR-10} \cite{krizhevsky2009learning} contains $60000$ color images of size $32 \times 32$ divided into $10$ classes, and each of class contains $6000$ images. Following conventional practise, the dataset is divided into $50000$ training images and $10000$ test images. We adopt classification accuracy as the evaluation metric. 
\end{itemize}

\subsection{Learning Accuracy}
In this section, we evaluate the training accuracy of our algorithm on both regression and classification tasks, and give the comparison with the non-private approach i.e., FedAvg \cite{McMahanMRHA17}, and other state-of-the-art defence methods, i.e., PPDL\cite{ShokriS15}, SPN \cite{FanNJZLCY20} and DBCL \cite{Wangshusen2019}. 

\subsubsection{Regression}
We evaluate the learning accuracy of regression tasks on the UBMD dataset in terms of different layers $L\in\{3,5,7\}$ and numbers of clients $K\in\{1,5,10\}$. Table \ref{reg} shows results of these schemes for the final converged model on testsets. From the table, the learning accuracy of our method elegantly aligns with that of the FedAvg \cite{McMahanMRHA17} under various settings and outperforms than other defence methods, i.e.,  PPDL\cite{ShokriS15}, SPN \cite{FanNJZLCY20} and DBCL \cite{Wangshusen2019}. The main reason is that our method can eliminate the perturbed noises to obtain the true updated parameters, and thus would not reduce the learning accuracy compared with the non-private model training. However, PPDL\cite{ShokriS15}, SPN \cite{FanNJZLCY20} and DBCL \cite{Wangshusen2019} cannot eliminate the perturbed noises, which will inevitably reduce the learning accuracy.

\begin{table*}[htb]
\begin{center}
\caption{MSE comparison of different defenses for Regression Tasks. Lower MSE Means Better Performance.}
\label{reg}
\begin{threeparttable}
\scalebox{0.95}{
\begin{tabular}{cccccccccccccccc}
\toprule
& \multicolumn{3}{c}{\textbf{FedAvg} \cite{McMahanMRHA17}}  & \multicolumn{3}{c}{\textbf{PPDL}} & \multicolumn{3}{c}{\textbf{DBCL}} & \multicolumn{3}{c}{\textbf{SPN}} & \multicolumn{3}{c}{\textbf{Ours}} \\
\cmidrule(ll){2-4} \cmidrule(lll){5-7} \cmidrule(lll){8-10} \cmidrule(lll){11-13} \cmidrule(lll){14-16} 
  Clients & $L=3$ & $L=5$ & $L=7$ &   $L=3$ & $L=5$ & $L=7$ & $L=3$ & $L=5$ & $L=7$ & $L=3$ & $L=5$ & $L=7$ & $L=3$ & $L=5$ & $L=7$ \\
   \midrule
   1 & \textbf{0.059} & \textbf{0.059} & \textbf{0.058} & 0.063& 0.062& 0.068& 0.062& 0.063& 0.062& 0.064& 0.060& 0.067& 0.061 & \textbf{0.059} & 0.060  \\
   5 & \textbf{0.079} & \textbf{0.079} & 0.086 & 0.084& 0.088& 0.085& 0.091& 0.086& 0.085& 0.082& 0.084& 0.088& 0.081 & 0.080 & \textbf{0.084} \\
   10 & \textbf{0.097} & \textbf{0.100} & 0.113 & 0.115& 0.108& 0.119 & 0.124& 0.118& 0.120& 0.126& 0.123& 0.115& 0.100 & 0.101 & \textbf{0.102} \\
 \bottomrule
\end{tabular}
}
    \scriptsize
        \begin{tablenotes}
        \item Some operations (e.g., dividing by random vectors) that are unavoidable in model training may cause precision errors, but the corresponding effects are negligible (see the results for the FedAvg and our scheme). 
        \end{tablenotes}
\end{threeparttable}
\end{center}
\vspace{-3mm}
\end{table*}

\subsubsection{Classification}
In this section, we evaluate our method for the classification task with ResNet20, ResNet32 and ResNet56 models on the LDC and CIFAR-10 datasets. All the training methods adopt the cross entropy as the loss function. The accuracy of converged models on testsets is shown in Table \ref{cls}. From the table, we can obtain the learning accuracy of our method elegantly aligns with the FedAvg \cite{McMahanMRHA17} and has significant advantage over PPDL\cite{ShokriS15}, SPN \cite{FanNJZLCY20} and DBCL \cite{Wangshusen2019}. 

Note that one of the appealing features of our method is to allow clients to train over noisy models without accuracy loss. To this end, we further empirically prove this claim by comparing the convergence process of our method against the FedAvg \cite{McMahanMRHA17}.  Specifically, we train ResNet20 and ResNet32 \cite{HeZRS16} models on Lesion Disease Classification  dataset \cite{abs-1803-10417} \cite{abs-1902-03368} for $200$ epochs for our method and the FedAvg, respectively. For the fairness of comparison, we use different random seeds for each run, and set the seeds for our method and the FedAvg to be the same. Then we demonstrate the mean and standard deviation of test accuracy after each epoch in Figure \ref{fig:additionalResult}. As shown in the figure, the convergence process of our method and the FedAvg is almost exactly consistent. Although the curves of our method and the FedAvg may be noisy in the early stage, they tend to stabilize and converge to similar results as the training proceeds.


\begin{table*}[tbh]
\centering
  \caption{Accuracy result for classification task on LDC and CIFAR10 dataset. }\label{cls}
  \label{tab:participant}
  \begin{threeparttable}
  \scalebox{0.72}{
  \begin{tabular}{ccccccccccccccccc}
    \toprule
     & & \multicolumn{3}{c}{\textbf{FedAvg}}  & \multicolumn{3}{c}{\textbf{PPDL}} &
    \multicolumn{3}{c}{\textbf{DBCL}} &
    \multicolumn{3}{c}{\textbf{SPN}} &
    \multicolumn{3}{c}{\textbf{Ours}} \\
     \cmidrule(ll){3-5} \cmidrule(lll){6-8} \cmidrule(lll){9-11} \cmidrule(lll){12-14}\cmidrule(lll){15-17}  
    & participants & $\text{ResNet20}$ & $\text{ResNet32}$ & $\text{ResNet56}$ &   $\text{ResNet20}$ & $\text{ResNet32}$ & $\text{ResNet56}$ & $\text{ResNet20}$ & $\text{ResNet32}$ & $\text{ResNet56}$ & $\text{ResNet20}$ & $\text{ResNet32}$ & $\text{ResNet56}$ & $\text{ResNet20}$ & $\text{ResNet32}$ & $\text{ResNet56}$\\
    \midrule
    \multirow{3}{*}{LDC} 
     & 1 & \textbf{69.42} & 70.96 & 72.79 & 66.73 & 67.21 & 67.34 & 67.33 & 67.56 & 67.83 & 68.25 & 68.63 & 69.31 & 69.33 & \textbf{71.05} & \textbf{72.83} \\
     & 5 & 69.27 & 70.78 & 71.61 & 67.41 & 67.45& 67.63 & 69.23& 69.85 & 70.46 & 68.47 & 68.63 & 68.94 & \textbf{69.29} & \textbf{70.84} & 71.55\\
     & 10 & 69.14 & \textbf{70.64} & 71.10 & 67.21 & 67.45 & 67.55 & 68.93& 70.42 & 70.95 &  68.62 & 68.74 & 69.31 & \textbf{69.22}  & 70.62 & \textbf{71.23} \\
     \midrule
     \multirow{3}{*}{CIFAR10} 
     & 1 & \textbf{92.31} & 92.59 & 93.18 & 88.75& 89.26& 89.32& 91.56& 91.73& 92.05 &  89.03 & 89.36 & 89.72 & 92.24 & \textbf{92.61} & \textbf{93.20} \\
     & 5 & \textbf{92.17} & 92.47 & \textbf{93.09} & 88.67& 89.16& 89.24& 91.53& 91.48& 91.69 & 88.45 & 89.63 & 90.12 & 92.14 & \textbf{92.50} & 93.05\\
     & 10 & \textbf{91.93} & 92.19 & 92.32 & 88.41& 88.53& 88.74& 91.25& 91.72& 92.03& 89.59 & 89.68 & 89.81 & 91.89  & \textbf{92.24} & \textbf{92.35} \\
    \bottomrule
    \end{tabular}
    }
    \end{threeparttable}
    \vspace{-3mm}
\end{table*}

\begin{figure*}[htbp]
\centering
\begin{minipage}[t]{0.3\textwidth}
\centering
\includegraphics[width=4.5cm]{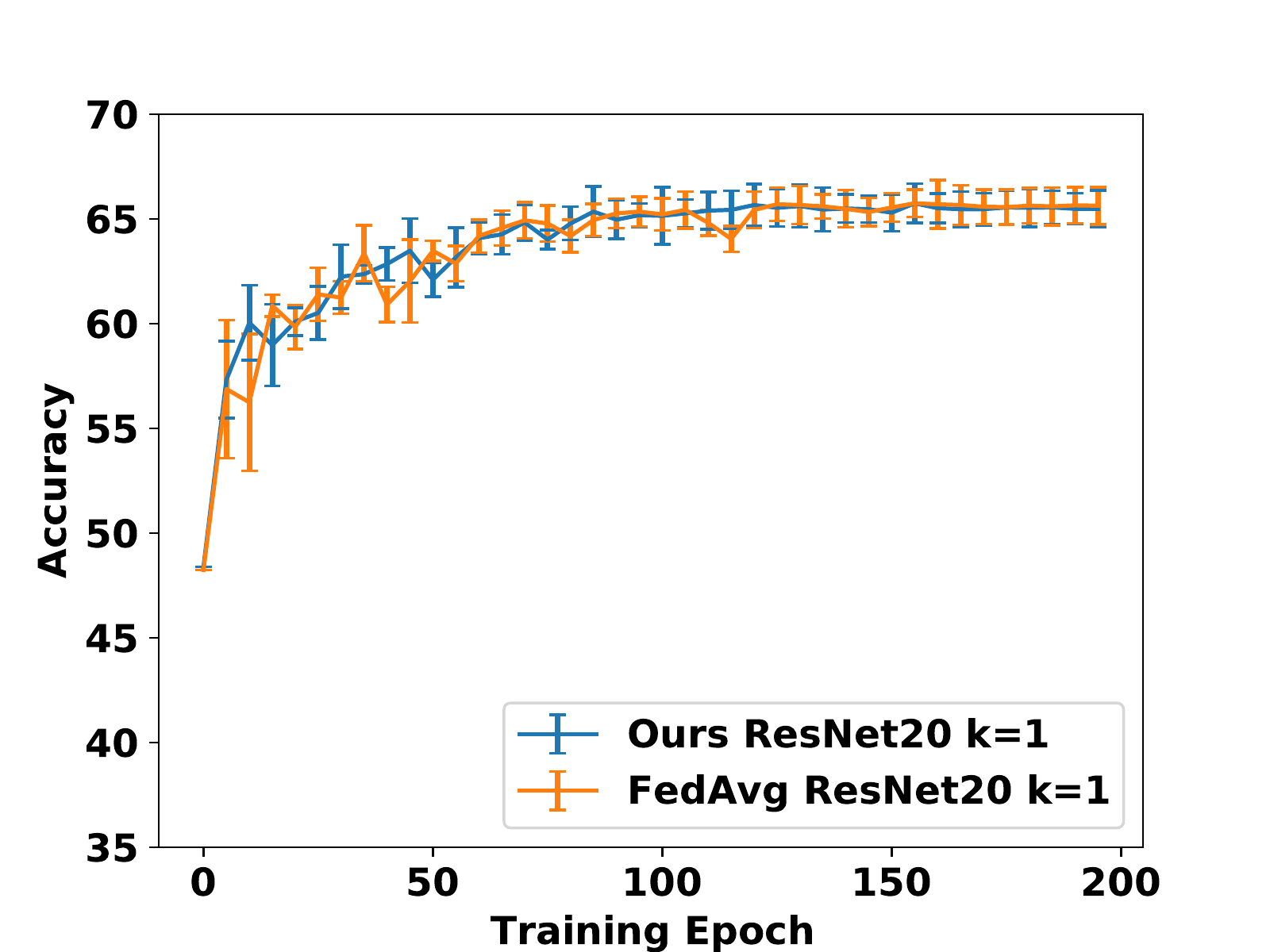}
\end{minipage}
\begin{minipage}[t]{0.3\textwidth}
\centering
\includegraphics[width=4.5cm]{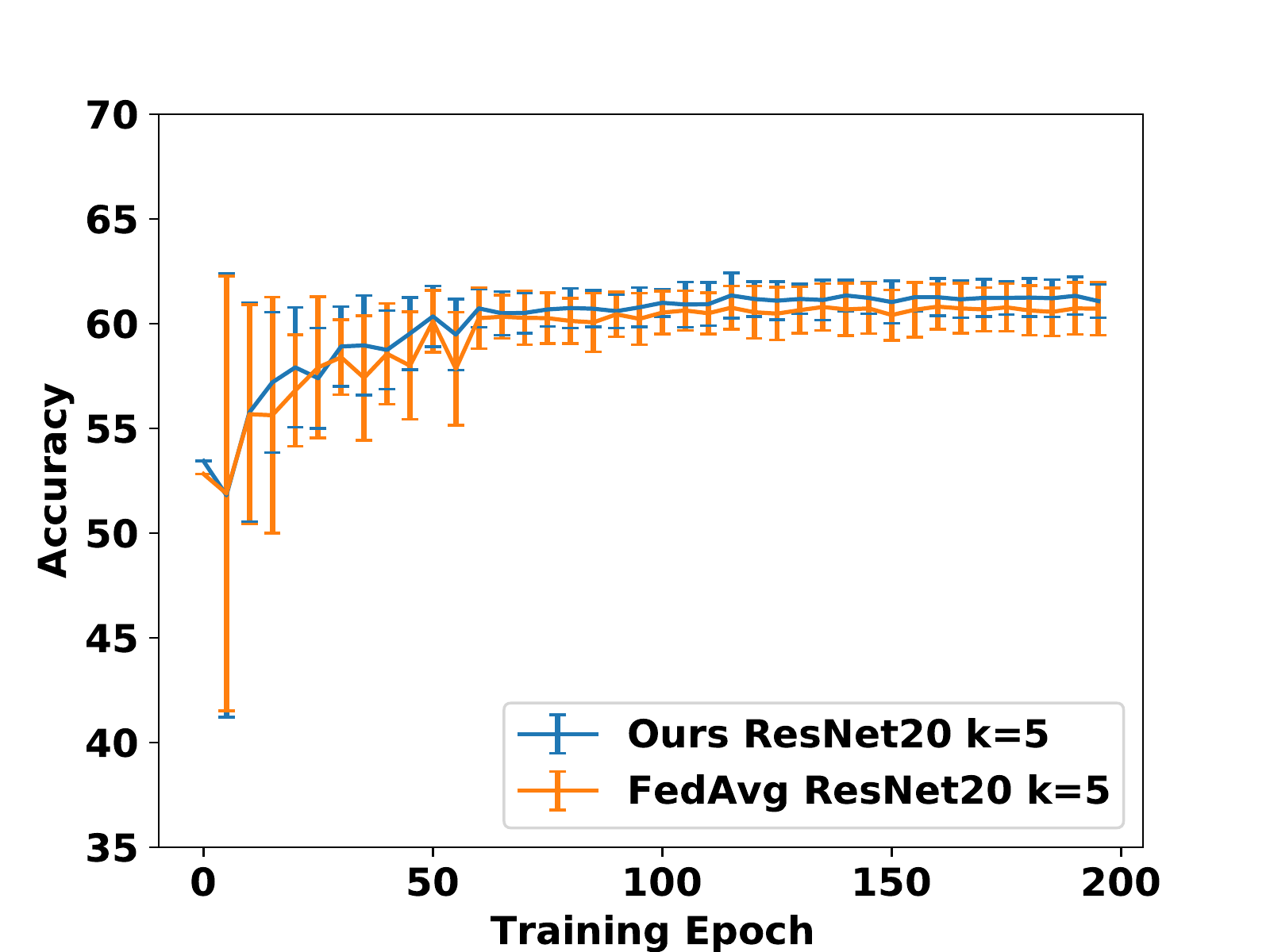}
\end{minipage}
\begin{minipage}[t]{0.3\textwidth}
\centering
\includegraphics[width=4.5cm]{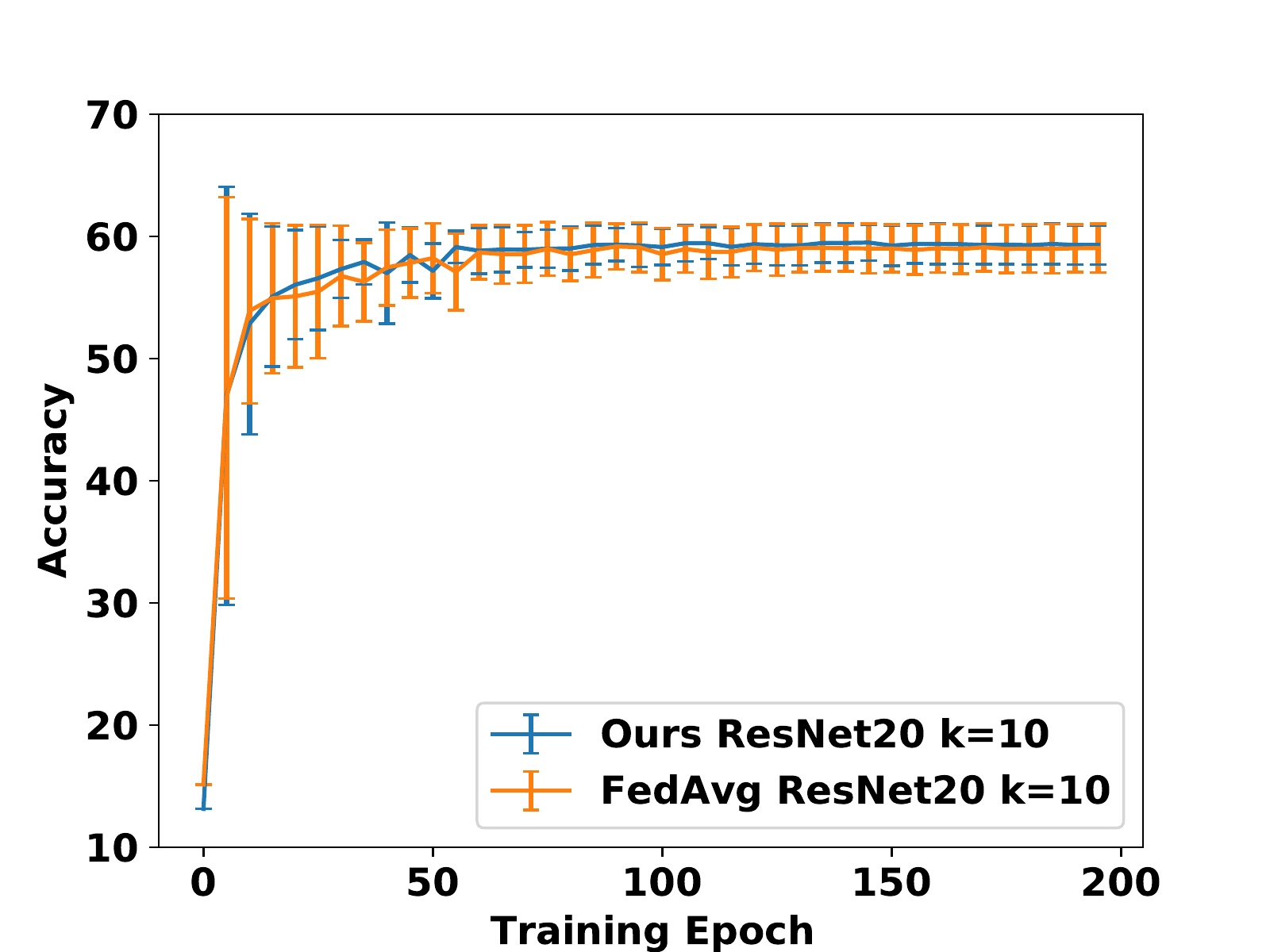}
\end{minipage}
\centering
\begin{minipage}[t]{0.3\textwidth}
\centering
\includegraphics[width=4.5cm]{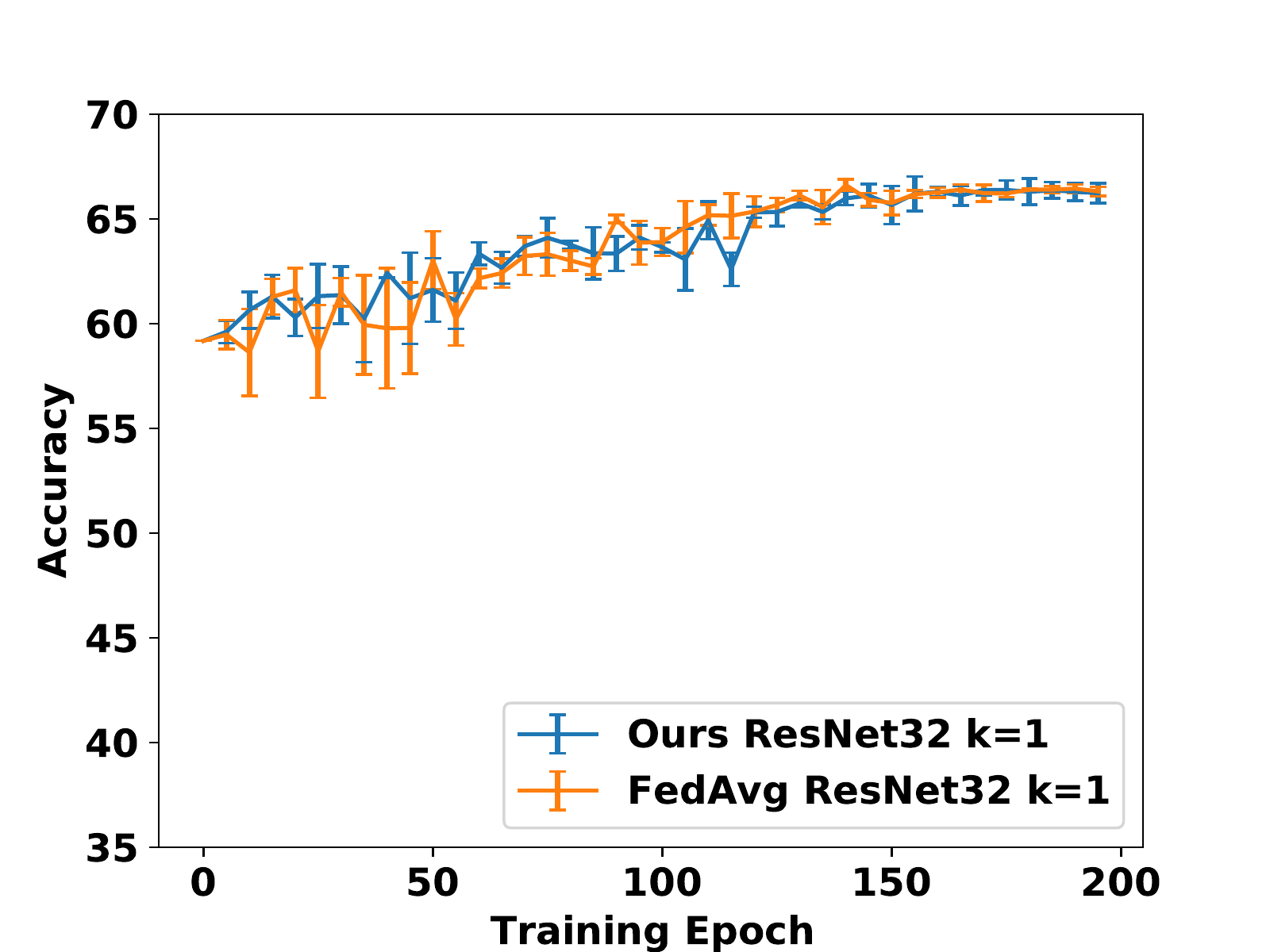}
\end{minipage}
\begin{minipage}[t]{0.3\textwidth}
\centering
\includegraphics[width=4.5cm]{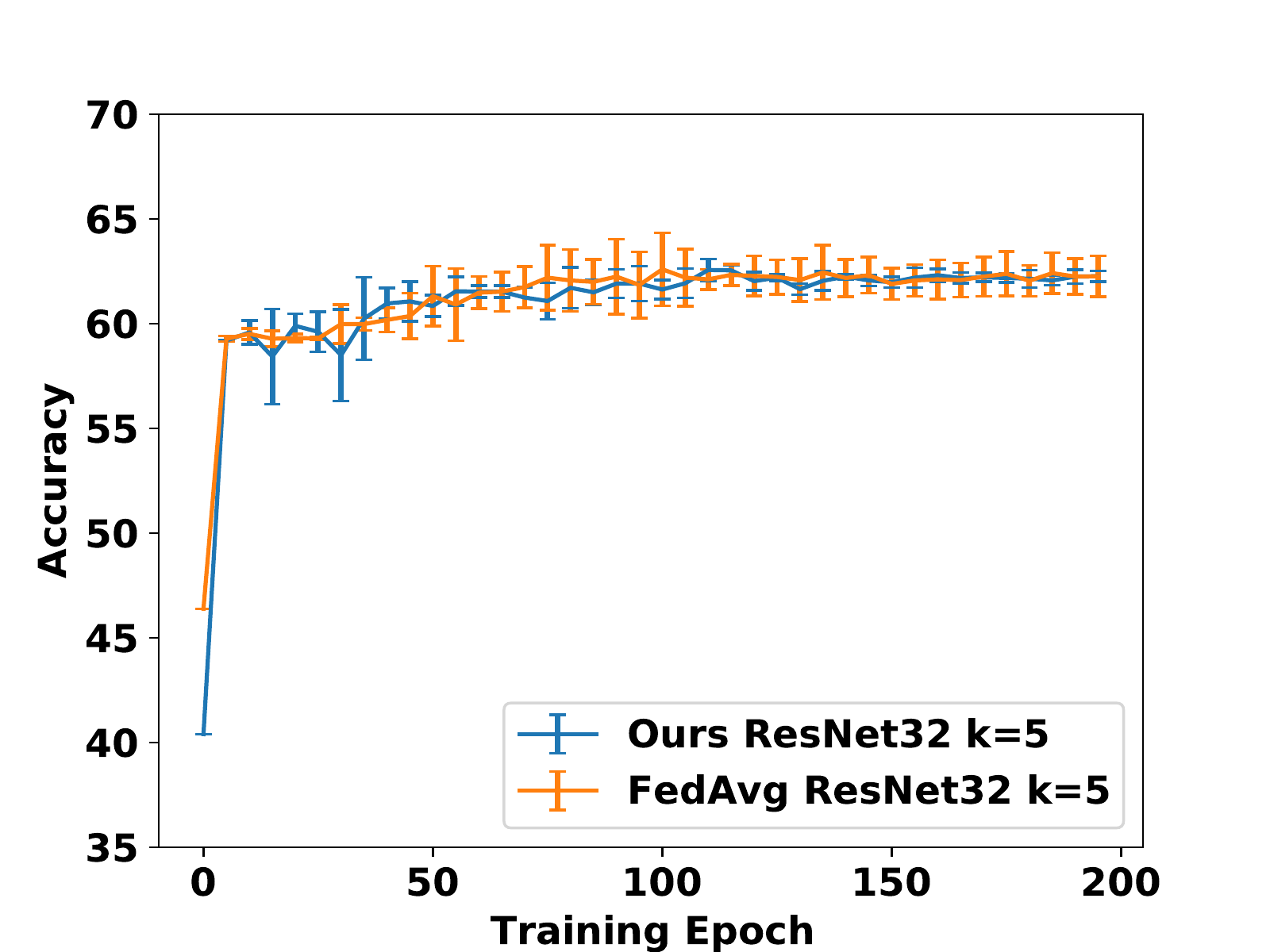}
\end{minipage}
\begin{minipage}[t]{0.3\textwidth}
\centering
\includegraphics[width=4.5cm]{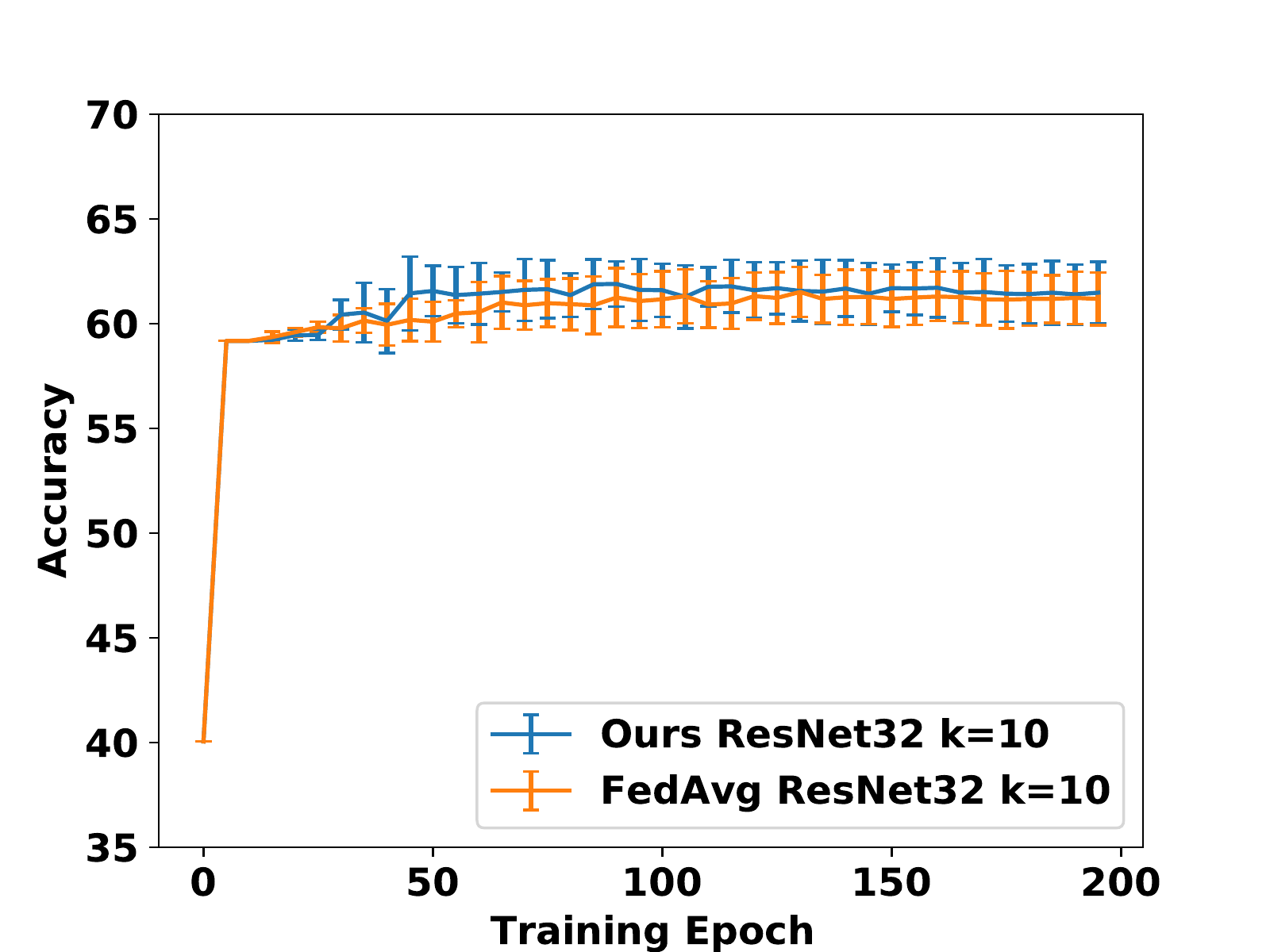}
\end{minipage}
\caption{Training curves for our method and FedAvg. The first and second rows demonstrate the results for ResNet20 and ResNet32 with $1$, $5$ and $10$ clients respectively. The error bars stand for standard deviation and are shown every $5$ epochs. Each experiment is repeated $5$ times. Please note that it is best to view in color.}
\label{fig:additionalResult}
\end{figure*}

\subsection{Experiments against Privacy Attacks}
In this section, we evaluate the performance of different defenses against state-of-the-art privacy attacks, i.e., reconstruction attacks and membership inference attacks, launched by curious clients. 


\subsubsection{Defense against membership inference Attacks}
In this part, we launch the membership inference attack that attempts to identify whether a data record is used during the training phase, against the FL training approaches. 

\noindent\textbf{Experiment setup:}
We implement the attacks based on \cite{NasrSH19}, which aims to tell if a certain data record was part of training set with a surrogate attack model. The attack model takes the final prediction or intermediate output of target models as input and the output has two classes ``\textit{Member}'' and ``\textit{Non-member}''. Following the settings in \cite{NasrSH19}, we assume the attacker has access to a fraction of (i.e., $50\%$) the training set and some non-member samples. In this case, to balance the training, we select half of each batch to include member instances and the other half non-member instances from the attacker's background knowledge, which prevents the attack model from a bias towards member or non-member instances.

 We initialized the weights using a normal distribution with mean $0$ and standard deviation $0.01$. The bias values of all layers are initialized with $0$. The batch size of all experiments is $64$. To train the attack model, we use the Adam optimizer with a learning rate of $0.0001$. We train attack models for $100$ epochs and pick the model with the highest testing accuracy, across all the $100$ epochs. We conduct experiments on {LDC} and CIFAR-10 dataset with ResNet56 models. We consider the most strict case where the number of clients is set as \textbf{2}. For a comprehensive comparison, we report the test accuracy and success rate for membership inference attacks by using the output of the last three layers. 

\noindent\textbf{Experiment result:} Tab. \ref{tab:best_white_models} shows the comparison results about membership inference attacks in our method, FedAvg \cite{McMahanMRHA17},  PPDL\cite{ShokriS15}, SPN \cite{FanNJZLCY20} and DBCL \cite{Wangshusen2019}. From the table, we can see that our method achieves the best learning accuracy while maintaining the strongest defence ability compared with other methods. More specifically, our method achieves attack success rate of around $50\%$, which means the clients can not infer any information from model outputs except blind guessing. Note that for the membership inference attack that infers a particular record is either a ``member'' or a ``non-member'', the worst attack success rate is $50\%$ (i.e., blind guessing). However, the attack success rate of other defence methods is around $60\%$, which means that clients have a certain advantage to attack successfully compared with blind guessing. 


\begin{table*}[htb] 
\centering
    \caption{\small The tracing attack results on LDC and CIFAR-10 dataset. We report the accuracy and attack success rate (ASR). We consider the cases where the malicious client launches attack using output of each of the last three layers.
    }
    \scalebox{0.9}{
    \begin{tabular}{ccccccc}
        \toprule
        & \textbf{Method} & \textbf{FedAvg}  & \textbf{PPDL} & \textbf{DBCL} & \textbf{SPN} & \textbf{Ours} \\
        \midrule
        \multirow{5}{*}{LDC} & \textbf{Accuracy} & $72.79$ & $67.34$ & $67.83$ & $69.31$ & $72.83$\\
        & \textbf{ASR (Last Layer)} & $70.21$ & 59.53& 64.75& 62.37 & $50.24$\\
        & \textbf{ASR (Second to Last)} & $68.50$ & 59.23& 53.44& 62.02& $49.96$ \\
        & \textbf{ASR (Third to Last)} & $66.45$ & 58.73& 52.71& 61.50& $49.87$ \\
        & \textbf{ASR (Last three)} & $72.33$ & 61.84 & 64.25& 62.41& $50.02$ \\
        \midrule
        \multirow{5}{*}{CIFAR10} & \textbf{Accuracy} & $93.18$ & $89.32$ & $92.05$ & $89.72$ &  $93.20$\\
        & \textbf{ASR (Last Layer)} & $77.31$ & 61.33& 66.44& 63.95& $50.24$\\
        & \textbf{ASR (Second to Last)} & $75.17$ & 60.59& 52.71& 61.72& $50.13$ \\
        & \textbf{ASR (Third to Last)} & $76.48$ & 60.83 & 51.62& 61.25& $49.83$ \\
        & \textbf{ASR (Last three)} & $77.56$ & 61.72 & 65.32& 63.80& $49.93$ \\
        \bottomrule
    \end{tabular}
    }
    \label{tab:best_white_models}
\end{table*}

\subsubsection{Defense against Reconstruction Attacks}
In this section, we launch two types of reconstruction attacks. On the one hand, recent research, dubbed GMI \cite{ZhangJP0LS20}, found that the private images from the training set of a certain model can be easily reconstructed with the help of a generative model (i.e. GAN \cite{ArjovskyCB17}). Thus, we first perform the experiments to show the defence ability of the attack that reconstruct the training images of certain category with the model output. On the other hand, Zhu et. al. \cite{ZhuLH19} presented the reconstruct attack, dubbed DLG, to recover training data from the local gradient information in FL. Thus, we also perform experiments to show the defence ability of the attack that reconstruct the training data from the local gradients. 

\noindent\textbf{Experiment setup: }
We evaluate the defensive ability of existing defence methods against GMI on LDC and CIFAR-10 datasets. Specifically, following \cite{ZhangJP0LS20}, we evenly split the training set based on the labels, and use one half as private set and the remainder as public set (the labels in the public set and private set have no overlaps). We evaluate the defence ability under the most strict setting: the attacker does not have any auxiliary knowledge about the private image, in which case he/she will recover the image from scratch. Similar to \cite{ZhangJP0LS20}, we adopt the following three metrics for evaluation: 1) Attack success rate (ASR): the prediction accuracy of the attacked model over the \textbf{reconstructed images}; 2) Feature Distance (Feat Dist): the feature distance between the reconstructed image and the centroid of the target class. The feature space is taken to be the output of the penultimate layer of the attacked model; and  3) K-Nearest Neighbor Distance (KNN Dist): We select the closest data point to the reconstructed image in the training set and the KNN Dist is set as their $l_2$ distance in the feature space. 

\noindent\textbf{Experiment results: }
For the defence performance of the reconstruction attack against the GMI, we present the comparison results of different defenses in Tab \ref{dp}. From the table, we can see that our method achieves dominant advantages in all three metrics compared with other related methods, i.e., FedAvg \cite{McMahanMRHA17},  PPDL\cite{ShokriS15}, SPN \cite{FanNJZLCY20} and DBCL \cite{Wangshusen2019}. Specifically, the ASR of our method is around $10\%$ for both LDC and CIFAR-10 datasets, which is almost the optimal defensive ability. Note that the number of categories in both two datasets is $10$, and thus the probability of random guessing the prediction is $10\%$. However, the ASR of the non-defensive method (i.e., the FedAvg \cite{McMahanMRHA17}) for the LDC and CIFAR-10 are $60\%$ and $68\%$, respectively. The ASR of other defensive methods (i.e., PPDL\cite{ShokriS15}, SPN \cite{FanNJZLCY20} and DBCL \cite{Wangshusen2019}) are around $50\%$, which demonstrates a relative poor defensive ability. Besides, we also provide some visualization of the reconstructed images from CIFAR-10 in Fig. \ref{recons}, where the reconstructed results of our method leaks almost zero information about the training dataset.

For the defence performance of the reconstruction attack against the DLG, we train ResNet20 models in CIFAR-100 on $5$ clients with our method. Besides, since the work \cite{ZhuLH19} suggests adopting the differential privacy technique to defend the DLG attack, we take the PPDL\cite{ShokriS15} as a comparison. In the PPDL\cite{ShokriS15}, the variance of added Gaussian noises are ranging from $10^{-4}$ to $10^{-1}$), respectively. Fig. \ref{recons_gradient} gives the visualization of reconstructed result on a randomly sampled image from the testset of CIFAR-100. From the figure, we can see that our method has the same level of defensive ability as the PPDL with the highest level of noise ($10^{-1}$).

\begin{table}[h]
\caption{Comparison of different defense methods against reconstruction attack.}\label{dp}
\centering
\resizebox{0.95\columnwidth}{!}{
\begin{threeparttable}
\begin{tabular}{ccccccc}
\toprule

               &         &   FedAvg   & PPDL & DBCL & SPN & Ours \\
\midrule
\multirow{4}{*}{LDC} 
    & Accuracy &   72.79   & 67.34  & 67.83 & 69.31 & 72.83    \\
& KNN Dist    & 682.23  & 1386.48   & 862.51  & 804.39 & 2248.56    \\
& Feat Dist    & 736.47   & 1683.99   & 1052.85  & 879.13 & 2203.71    \\
& ASR    & 60.22   & 33.70   & 51.66  & 55.37 & 11.49    \\
\midrule
\multirow{4}{*}{CIFAR10} 
    & Accuracy &   93.18   & 89.32  & 92.05 & 89.72 & 93.20    \\
& KNN Dist    & 364.71  & 781.43   & 662.82  & 721.66 & 1866.47    \\
& Feat Dist    & 405.61   & 723.90   & 703.98  & 768.36 & 1953.94    \\
& ASR    & 68.41   & 49.62   & 58.34  & 52.15 & 10.63    \\
\bottomrule
\end{tabular}

  \scriptsize
        \begin{tablenotes}
        \item ASR: the prediction accuracy of the attacked model over the reconstructed images. Feat Dist: the feature distance between the reconstructed image and the centroid of the target class. KNN Dist: the $l_2$ distance of 
        the closest data point to the reconstructed image. 
        \end{tablenotes}
\end{threeparttable}
}
\end{table}

\begin{figure}[htb] 
\centering 
\includegraphics[width=0.9\linewidth]{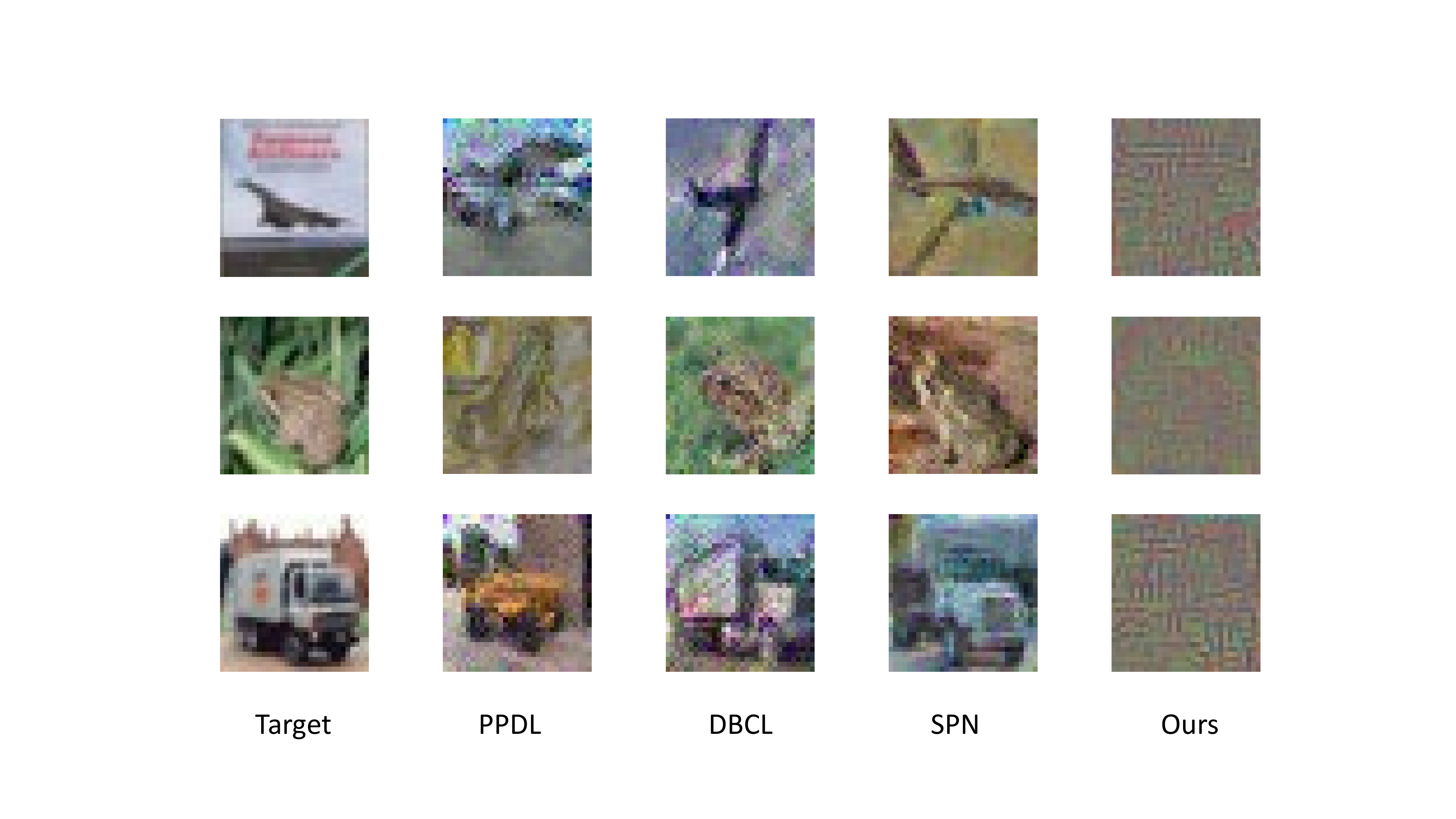} 
\vskip -0.25in
\caption{Reconstructed results of GMI against different defenses.} 
\label{recons}
\vskip -0.15in
\end{figure}

\begin{figure}[htb] 
\centering 
\includegraphics[width=0.9\linewidth]{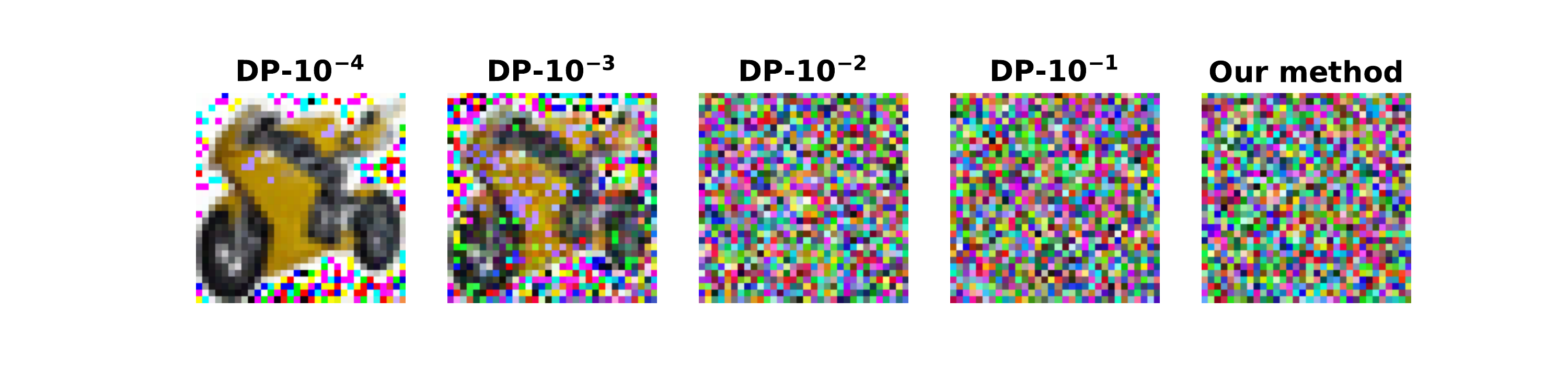} 
\vskip -0.25in
\caption{Reconstructed results of DLG against different defenses.} 
\label{recons_gradient}
\vskip -0.15in
\end{figure}

\section{Related work}\label{sec:related}
Federated learning (FL) was formally introduced by Google in 2016 \cite{KonecnyMYRSB16} to address data privacy in machine learning. However, recent works \cite{ZhuLH19, NasrSH19} proved that the original FL schemes still face the risk of privacy leakage. Specifically, two types of inference attacks called reconstruction and membership inference attacks have been widely leveraged to obtain local training data of clients. To defend these two privacy attacks, some compression methods such as selective gradients sharing \cite{ShokriS15}, reducing dimensionality \cite{abs-1912-11464} and dropout \cite{WagerWL13} were presented to decrease information disclosure, thereby defending against privacy attacks. However, these methods were proved to be almost ineffective for the defence ability or had a negative impact on the quality of the FL \cite{FanNJZLCY20}. After that, due to the high efficiency and easy integration, differential privacy technique \cite{Dwork06} was widely adopted in FL to improve the defence ability \cite{AbadiCGMMT016, McMahanRT018, PathakRR10, WeiLDMYFJQP20}. For example, \cite{AbadiCGMMT016} demonstrated how to maintain data privacy by adding Gaussian noise to shared gradients during the training of deep neural network. \cite{ZhuLH19} also suggested adding Gaussian noise to shared gradients to defend reconstruction attacks. However, differential privacy-based methods sacrifice model accuracy in exchange for privacy \cite{XuLL0L20}, which is not suitable in some FL applications requiring high model accuracy \cite{abs-1912-04977, ZhangLX00020}. Following the idea of differential privacy that injects random noises to perturb the shared parameters, the matrix sketching method \cite{Mahoney11, ClarksonW13} was considered in FL to perturb the global model parameters. Besides, Fan et. al. \cite{FanNJZLCY20} leveraged element-wise adaptive gradients perturbations to defend reconstruction and membership inference attacks. 

Unfortunately, these noises perturbation methods cannot achieve high learning accuracy and strong defence ability for reconstruction and membership inference attacks at the same time. Thus, in this paper, we try to design a novelty noise perturbation method that can achieve strong defence ability without sacrificing the learning accuracy. Note that there are other  cryptography technologies, such as secure multi-party  computation (MPC) \cite{WangCL17} and homomorphic encryption (HE) \cite{Brakerski12} were proposed to address privacy risks in FL \cite{YangLCT19, LiuCV18, abs-1901-08755, PhongAHWM18}. These schemes can provide strong privacy protection, but often incur significantly more demanding computational and communication cost while cannot be efficiently implemented in practice \cite{FanNJZLCY20, abs-1912-04977, ZhangLX00020}. In this paper, our work is similar to the differential privacy, and thus does not consider the MPC and HE based schemes therein and refer readers to \cite{TanuwidjajaCK19, YangLCT19}.

\section{Conclusion}\label{sec:conclusion}
In this paper, we have presented an efficient model perturbation method in federated learning to defense the state-of-the-art privacy attacks launched by honest-but-curious clients. Specifically, the server perturbs the global model parameters by adding random selected noises before broadcasting them to clients, which prevents clients from obtaining true model parameters including the true local gradients. Therefore, it can defense two famous privacy attacks, i.e., reconstruction attack and membership inference attack. Meanwhile, unlike the differential privacy that cannot remove the added random noises, the proposed method ensures that the server can remove the added random noises from the aggregated local gradients, and thus would not reduce the learning accuracy. Extensive experiments about both regression and classification tasks demonstrate that the model accuracy of our scheme are almost the same as the plain training, and better than the state-of-the-art privacy defense schemes. Besides, extensive experiments about both reconstruction and membership inference attacks show that the defense ability of our scheme significantly outperforms than the state-of-the-art privacy defense schemes. 


\bibliographystyle{IEEEtranS}
\bibliography{example_paper}


\appendices
\section{Proof of Theorem \ref{theorem:output}}\label{proof_theory0}

Based on Eq. \eqref{eq:para_cons1}, we can deduce that
\begin{equation*}
\left\{
\begin{aligned}
&R^{(1)}= D_{\bm r^{(1)}}E^{(1)} \textnormal{ and } R^{(L)} D_{\bm r^{(L-1)}}=E^{(L)},\\
&R^{(l)}D_{\bm r^{(l - 1)}}=  D_{\bm r^{(l)}}E^{(l)}, \textnormal{ for }2\leq l\leq L-1,
\end{aligned}\right.
\end{equation*}
 where $D_{\bm r^{(1)}}$ is the $n_l \times n_{l}$ diagonal matrix whose main diagonal is $\bm r^{(1)}$ and $E^{(l)}$ is the $n_l \times n_{l-1}$ matrix whose entries are all 1s, for $l=1,2, \dots, L$.
 Next, we first prove Eq. \eqref{eq:l_out} by mathematical induction. Specifically, when $l=1$, we can obtain
\begin{small}
\begin{align*}
    \hat {\bm y}^{(1)} &= ReLU\left(\widehat W^{(1)} \bm x\right)=  ReLU\left(\big(R^{(1)} \circ W^{(1)}\big) \bm x\right) \\
    &=ReLU\left(\big(D_{\bm r^{(1)}}E^{(1)} \circ W^{(1)}\big) \bm x\right)\\
    &\stackrel{(a)}=ReLU\left(D_{\bm r^{(1)}}\big(E^{(1)} \circ W^{(1)}\big) \bm x\right)\\
    &=  ReLU\left(D_{\bm r^{(1)}} W^{(1)} \bm x\right)\stackrel{(b)}=  ReLU\left(\bm r^{(1)}\circ (W^{(1)} \bm x)\right)\\
    & \stackrel{(c)}=\bm r^{(1)}\circ ReLU( W^{(1)} \bm x)
    =\bm r^{(1)} \circ \bm y^{(1)},
\end{align*}
\end{small}
where the above equations $(a)$ and $(b)$ follow from the properties of Hadamard product (See Definition \ref{hadamard}), and the equation $(c)$ follows from $\bm r^{(1)} \in R_{>0}^{n_1}$.
Then, for $2 \leq l \leq L-1$, assuming $\hat {\bm y}^{(l - 1)}= \bm r^{(l - 1)} \circ \bm y^{(l - 1)}$ by induction. Then, we have
\begin{small}
\begin{align*}
    \hat {\bm y}^{(l)} &= ReLU\left(\widehat W^{(l)} \hat {\bm y}^{(l - 1)}\right)\\
    &= ReLU\left(\big(R^{(l)} \circ W^{(l)}\big) \big(\bm r^{(l - 1)} \circ \bm y^{(l - 1)}\big)\right) \\
    &=ReLU\left(\big(R^{(l)} \circ W^{(l)}\big) D_{\bm r^{(l - 1)}}  \bm y^{(l - 1)}\right)\\
    &=ReLU\left(\big( (R^{(l)}D_{\bm r^{(l - 1)}}) \circ W^{(l)}\big) \bm y^{(l - 1)}\right) \\
    &=ReLU\left(\big( D_{\bm r^{(l )}}E^{(l)} \circ W^{(l)} \big) \bm y^{(l - 1)}\right)\\
   &=ReLU\left(\big( D_{\bm r^{(l )}}(E^{(l)} \circ W^{(l)}) \big) \bm y^{(l - 1)}\right)\\
   &=ReLU\left( D_{\bm r^{(l )}}W^{(l)} \bm y^{(l - 1)}\right)= ReLU\left(\bm r^{(l )} \circ (W^{(l)} \bm y^{(l - 1)})\right)\\
   & = \bm r^{(l )} \circ ReLU\left(W^{(l)} \bm y^{(l - 1)}\right)=\bm r^{(l)} \circ \bm y^{(l)}.
\end{align*}
\end{small}
Then, we prove Eq. \eqref{eq:final_out} as follows.
\begin{small}
\begin{align}
\nonumber    \hat {\bm y}^{(L)} &= \widehat W^{(L)} \hat {\bm y}^{(L - 1)} = (R^{(L)} \circ W^{(L-1)} + R^{(a)}) \hat {\bm y}^{(L-1)} \\
\nonumber    &= (R^{(L)} \circ W^{(L-1)})(\bm r^{(L - 1)} \circ \bm y^{(L - 1)}) + R^{(a)}\hat {\bm y}^{(L-1)}  \\
\nonumber    &= (R^{(L)} D_{\bm r^{(L-1)}}) \circ W^{(L)} \bm y^{(L - 1)}) + R^{(a)}\hat {\bm y}^{(L-1)}\\
\nonumber    &=  W^{(L)} \bm y^{(L - 1)} + R^{(a)}\hat {\bm y}^{(L-1)}= \bm y^{(L)} + \alpha \bm \gamma \circ \bm r^{(a)}\\
\nonumber    &= \bm y^{(L)} +  \alpha \bm{r}.
\end{align}
\end{small}

\section{Proof of Theorem \ref{gradient}}\label{proof_Theory1}
Before giving the proof, we recall some notations about the derivatives of vector-valued functions. Specifically, for any vectors $\bm f \in \mathbb{R}^{m}$ and $\bm x \in \mathbb{R}^{n}$, the partial derivative $\frac{\partial \bm f}{\partial \bm x}$ of $\bm f$ with respect to $\bm x$ is an $m \times n$ matrix, whose $(i,j)$-th entry is given as 
\[\Big(\frac{\partial \bm f}{\partial \bm x}\Big)_{ij}=\frac{\partial f_i}{\partial x_j}.\]
Moreover, when $\bm x$ is a $u \times v$ matrix, we can regard $\bm x$ as a vector of $\mathbb{R}^{uv}$, then $\frac{\partial \bm f}{\partial \bm x}$ is also well-defined.

According to Theorem \ref{theorem:output}, the prediction is computed as $\bm{\hat{y}}^{(L)} =\bm y^{(L)} +\alpha \bm{r}$, and thus the perturbed loss function shown in Eq. \eqref{noisy loss function} is deduced as 
\[
\mathcal{\widehat{L}}\left(\widehat{W}; (\bm{x},\bm{\bar{y}})\right)=\frac{1}{2}\parallel\bm y^{(L)} +\alpha \bm{r}- \bar{\bm y}\parallel^{2}_{2}.
\]
Then, the derivative of the loss with respect to the prediction is 
\begin{small}
\[\frac{\partial \mathcal{\widehat L}\left(\widehat{W}; (\bm{x},\bm{\bar{y}})\right)}{\partial \hat {\bm y}^{(L)}}=\left(\hat {\bm y}^{(L)}-\bar{\bm y}\right)^{T}=\frac{\partial \mathcal{ L}\left(W; (\bm{x},\bm{\bar{y}})\right)}{\partial {\bm y}^{(L)}}+ \alpha \bm r^{T},\]
\end{small}
By the chain rule, we can derive that
\begin{small}
\begin{eqnarray*}
 \frac{\partial \mathcal{\widehat  L}(\widehat{W}; (\bm{x},\bm{\bar{y}}))}{\partial \widehat {W}^{(l)}}
 &=&\frac{\partial \mathcal{\widehat L}(\widehat{W}; (\bm{x},\bm{\bar{y}}))}{\partial \hat {\bm y}^{(L)}}
  \frac{\partial \hat {\bm y}^{(L)}}{\partial  \widehat {W}^{(l)}}\\
 & =&\left(\frac{\partial \mathcal{L}(W; (\bm{x},\bm{\bar{y}}))}{\partial {\bm y}^{(L)}} +  \alpha \bm r^{T} \right)
  \frac{\partial \hat {\bm y}^{(L)}}{\partial  \widehat {W}^{(l)}} \\
    &=&\frac{\partial \mathcal{L}(W; (\bm{x},\bm{\bar{y}}))}{\partial {\bm y}^{(L)}} \frac{\partial \hat {\bm y}^{(L)}}{\partial  \widehat {W}^{(l)}}+  \alpha \bm r^{T}
  \frac{\partial \hat {\bm y}^{(L)}}{\partial  \widehat {W}^{(l)}} \\
  &=&\frac{\partial \mathcal{L}(W; (\bm{x},\bm{\bar{y}}))}{\partial {\bm y}^{(L)}}\left( \frac{\partial  {\bm y}^{(L)}}{\partial  \widehat {W}^{(l)}}+\frac{\partial (\alpha \bm{r} )}{\partial \widehat {W}^{(l)}}\right)\\
  &&+  \alpha \bm r^{T}
  \frac{\partial \hat {\bm y}^{(L)}}{\partial  \widehat {W}^{(l)}} \\
  &=& \frac{\partial \mathcal{L}(W; (\bm{x},\bm{\bar{y}}))}{\partial \widehat {W}^{(l)}} +\alpha \bm r^{T}
  \frac{\partial \hat {\bm y}^{(L)}}{\partial  \widehat {W}^{(l)}}\\
  &&+
  \left(\frac{\partial \mathcal{\widehat L}(\widehat{W}; (\bm{x},\bm{\bar{y}}))}{\partial \hat{{\bm y}}^{( L)}} -  \alpha \bm r^{T}\right)\frac{\partial (\alpha \bm r )}{\partial \widehat {W}^{(l)}}\\
  &\stackrel{(d)}=& \frac{\partial \mathcal{L}(W; (\bm{x},\bm{\bar{y}}))}{\partial \widehat {W}^{(l)}} 
  - \upsilon \alpha \frac{\partial  \alpha}{\partial \widehat {W}^{(l)}} +\\
  &&\bm r^{T} \left( \alpha\frac{\partial \hat {\bm y}^{(L)}}{\partial \widehat {W}^{(l)}}
  + \Big(\frac{\partial \mathcal{\widehat L}(\widehat{W}; (\bm{x},\bm{\bar{y}}))}{\partial \hat{\bm y}^{(L)}}\Big)^{T}\frac{\partial  \alpha}{\partial \widehat {W}^{(l)}}\right) \\
  &=& \frac{1}{R^{(l)}} \circ \frac{\partial \mathcal{L}(W; (\bm{x},\bm{\bar{y}}))}{\partial {W}^{(l)}} +
  \bm{r}^{T}  \bm \sigma^{(l)} - \upsilon \bm \beta^{(l)},
\end{eqnarray*}
\end{small}
where $(d)$ follows from the fact that $\frac{\partial \mathcal{\widehat L}(\widehat{W}; (\bm{x},\bm{\bar{y}}))}{\partial \hat{\bm y}^{(L)}} \bm r=\bm r^{T}\left(\frac{\partial \mathcal{\widehat L}(\widehat{W}; (\bm{x},\bm{\bar{y}}))}{\partial \hat{\bm y}^{(L)}}\right)^{T} \in \mathbb{R}$.

\section{Proof of Theorem \ref{modelupdate}}\label{proof_Theory2}
According to Theorem \ref{gradient}, we can derive that
\begin{small}
\begin{eqnarray*}
\nabla F(\widehat{W}^{(l)}, \mathcal{D}_{k})&=&\frac{1}{|\mathcal{D}_{k}|}\sum_{(\bm{x}_{i}, \bm{\bar{y}}_{i})\in \mathcal{D}_{k}}\frac{\partial \mathcal{\widehat{L}}\left(\widehat{W}; (\bm{x}_{i}, \bm{\bar{y}}_{i})\right)}{\partial \widehat{W}^{(l)}} \\
&=&\frac{1}{|\mathcal{D}_{k}|}\sum_{(\bm{x}_{i}, \bm{\bar{y}}_{i})\in \mathcal{D}_{k}}\bigg(\frac{1}{R^{(l)}} \circ \frac{\partial \mathcal{L}\left(W; (\bm{x}_{i}, \bm{\bar{y}}_{i})\right)}{\partial {W}^{(l)}} \\
&&+\bm{r}^{T} \bm \sigma^{(l)}_{(\bm{x}_{i}, \bm{\bar{y}}_{i})} - \upsilon \bm \beta^{(l)}_{(\bm{x}_{i}, \bm{\bar{y}}_{i})}\bigg)\\
  &=& \frac{1}{R^{(l)}} \circ \nabla F(W^{(l)}, \mathcal{D}_{k})+\bm{r}^{T} \bm \sigma^{(l)}_{k}-\upsilon \bm \beta^{(l)}_{k}.
\end{eqnarray*}
\end{small}
Thus, we can obtain that 
\begin{small}
\begin{eqnarray*}
\nabla F(W^{(l)})&=& R^{(l)}\circ\left(\nabla F(\widehat{W}^{(l)})-\left(\sum_{s=1}^{m}\bm\gamma_{I_s}\bm \sigma^{(l)}_{s}\right)+\upsilon \bm{ \beta}^{(l)}\right)\\
& \stackrel{(e)}= &  R^{(l)}\circ \sum^{K}_{k=1}\frac{|\mathcal{D}_{k}|}{|\mathcal{D}|}\left(\nabla F(\widehat{W}^{(l)}, \mathcal{D}_{k})-\bm{r}^{T}\bm \sigma_{k}^{(l)}+\upsilon\bm{\beta}_{k}^{(l)}\right)\\
&=&R^{(l)}\circ \sum^{K}_{k=1}\frac{|\mathcal{D}_{k}|}{|\mathcal{D}|}\Bigg(\frac{1}{R^{(l)}} \circ \nabla F(W^{(l)}, \mathcal{D}_{k})+\bm{r}^{T} \bm \sigma^{(l)}_{k}\\
&&-\upsilon \bm \beta^{(l)}_{k}-\bm{r}^{T}\bm \sigma_{k}^{(l)}+\upsilon\bm{\beta}_{k}^{(l)}
\Bigg)\\
&=&R^{(l)}\circ\frac{1}{R^{(l)}} \circ \sum^{K}_{k=1}\frac{|\mathcal{D}_{k}|}{|\mathcal{D}|} \nabla F(W^{(l)}, \mathcal{D}_{k}) \\
 &=& \sum^{K}_{k=1}\frac{|\mathcal{D}_{k}|}{|\mathcal{D}|}\nabla F(W^{(l)}, \mathcal{D}_{k}),
\end{eqnarray*}
\end{small}
where the equation $(e)$ satisfies the following deduction
\begin{small}
\begin{eqnarray*}
  \sum_{s=1}^{m}\bm\gamma_{I_s}\widetilde{\bm \sigma}^{(l)}_{s} &=&  \sum_{s=1}^{m}\bm\gamma_{I_s}\left(\sum^{K}_{k=1}\frac{|\mathcal{D}_{k}|}{|\mathcal{D}|}\widetilde{\bm \sigma}^{(l)}_{k,s}\right) \\
   &=&  \sum^{K}_{k=1}\frac{|\mathcal{D}_{k}|}{|\mathcal{D}|}\left(\sum_{s=1}^{m}\bm\gamma_{I_s}(\bm r^{(a)}_{|I_s})^{T}\bm{\widehat{\sigma}}^{(l)}_{{k}_{| I_s}}\right) \\
  &=&  \sum^{K}_{k=1}\frac{|\mathcal{D}_{k}|}{|\mathcal{D}|}\bm r^{T}\widehat{\bm \sigma}^{(l)}_{k} =\sum^{K}_{k=1}\frac{|\mathcal{D}_{k}|}{|\mathcal{D}|}\bm r^{T}\bm \sigma_{k}^{(l)}.
\end{eqnarray*}
\end{small}



\end{document}